\documentclass[english,reqno]{amsart}

\usepackage[margin=2.5cm]{geometry}

\input{preamble}

\newcommand{\Ebb}{\mathbb{E}}
\newcommand{\Lbb}{\mathbb{L}}
\newcommand{\Nbb}{\mathbb{N}}
\newcommand{\Pbb}{\mathbb{P}}

\newcommand{\Rbb}{\mathbb{R}}

\newcommand{\Xbb}{\mathbb{X}}
\newcommand{\Ybb}{\mathbb{Y}}

\newcommand{\Abf}{\mathbf{A}}
\newcommand{\bbf}{\mathbf{b}}

\newcommand{\cbf}{\mathbf{c}}
\newcommand{\Cbf}{\mathbf{C}}

\newcommand{\Gbf}{\mathbf{G}}

\newcommand{\Ibf}{\mathbf{I}}

\newcommand{\kbf}{\mathbf{k}}
\newcommand{\Kbf}{\mathbf{K}}

\newcommand{\Lbf}{\mathbf{L}}

\newcommand{\tbf}{\mathbf{t}}

\newcommand{\Ubf}{\mathbf{U}}

\newcommand{\wbf}{\mathbf{w}}

\newcommand{\xbf}{\mathbf{x}}
\newcommand{\Xbf}{\mathbf{X}}

\newcommand{\Ybf}{\mathbf{Y}}

\newcommand{\ep}{\epsilon}

\newcommand{\CalC}{{\mathcal{C}}}
\newcommand{\CalD}{{\mathcal{D}}}
\newcommand{\CalB}{{\mathcal{B}}}

\newcommand{\CalF}{{\mathcal{F}}}

\newcommand{\CalH}{{\mathcal{H}}}
\newcommand{\CalI}{{\mathcal{I}}}

\newcommand{\CalL}{{\mathcal{L}}}

\newcommand{\CalN}{{\mathcal{N}}}
\newcommand{\CalO}{{\mathcal{O}}}
\newcommand{\CalQ}{{\mathcal{Q}}}
\newcommand{\CalP}{{\mathcal{P}}}

\newcommand{\CalT}{{\mathcal{T}}}

\newcommand{\scL}{\mathscr{L}}

\newcommand{\nrm}[1]{\left\Vert {#1} \right\Vert}

\newcommand{\supp}[1]{\text{supp}\left(#1\right)}

\newcommand{\ind}[1]{\mathbbm{1}_{#1}}

\newcommand{\argmin}{\text{argmin}}

\newcommand{\lan}{\left\langle}
\newcommand{\ran}{\right\rangle}

\newcommand{\wstar}{\wbf^\star}
\newcommand{\what}{{\widehat{\wbf}}}

\newtheorem{thm}{Theorem}[section]
\newtheorem{lemm}{Lemma}[section]

\theoremstyle{definition}

\newtheorem{assump}{Assumptions}
\newtheorem{rmrk}{Remark}[section]

\usepackage{array}
\begin{document}

%
%

\title{Learning Mean-Field Equations from Particle Data Using WSINDy}
\author{Daniel A. Messenger, David M. Bortz}
\email{daniel.messenger@colorado.edu, dmbortz@colorado.edu\footnote{Department of Applied Mathematics, University of Colorado Boulder, 11 Engineering Dr., Boulder, CO 80309, USA.}}
\maketitle

\begin{abstract}
We develop a weak-form sparse identification method for interacting particle systems (IPS) with the primary goals of reducing computational complexity for large particle number $N$ and offering robustness to either intrinsic or extrinsic noise. In particular, we use concepts from mean-field theory of IPS in combination with the weak-form sparse identification of nonlinear dynamics algorithm (WSINDy) to provide a fast and reliable system identification scheme for recovering the governing stochastic differential equations for an IPS when the number of particles per experiment $N$ is on the order of several thousand and the number of experiments $M$ is less than 100. This is in contrast to existing work showing that system identification for $N$ less than 100 and $M$ on the order of several thousand is feasible using strong-form methods. We prove that under some standard regularity assumptions the scheme converges with rate $\CalO(N^{-1/2})$ in the ordinary least squares setting and we demonstrate the convergence rate numerically on several systems in one and two spatial dimensions. Our examples include a canonical problem from homogenization theory (as a first step towards learning coarse-grained models), the dynamics of an attractive-repulsive swarm, and the IPS description of the parabolic-elliptic Keller-Segel model for chemotaxis. 

\end{abstract}


\section{Problem Statement}\label{sec:problem_statement}

Consider a particle system $\Xbf_t = (X^{(1)}_t,\dots,X^{(N)}_t) \in \Rbb^{Nd}$ where on some fixed time window $t\in[0,T]$, each particle $X^{(i)}_t\in \Rbb^d$ evolves according to the overdamped dynamics
\begin{equation}\label{dXt}
dX_t^{(i)} = \left(-\nabla K*\mu^N_t\left(X^{(i)}_t\right)-\nabla V\left(X^{(i)}_t\right) \right) dt + \sigma(X_t^{(i)})\, dB_t^{(i)}
\end{equation}
with initial data $X_0^{(i)}$ each drawn independently from some probability measure $\mu_0\in \CalP_p(\Rbb^d)$, where $\CalP_p(\Rbb^d)$ is the space probability measures on $\Rbb^d$ with finite $p$th moment\footnote{We define the  $p$th moment of a probability measure $\mu$ for $p>0$ by $\int_{\Rbb^d}|x|^pd\mu(x)$.}. Here, $K$ is the {\it interaction potential} defining the pairwise forces between particles, $V$ is the {\it local potential} containing all exogenous forces, $\sigma$ is a {\it diffusivity}, and $\left(B^{(i)}_t\right)_{i=1,\dots,N}$ are independent Brownian motions each adapted to the same filtered probability space $(\Omega, \CalB, \Pbb, (\CalF_t)_{t\geq 0})$.  The {\it empirical measure} is defined
\[\mu^N_t := \frac{1}{N}\sum_{i=1}^N \delta_{X^{(i)}_t},\]
and the convolution $\nabla K* \mu^N_t$ is defined 
\[\nabla K* \mu^N_t(x) = \nabla\int_{\Rbb^d} K(x-y)\,d\mu^N_t(y)=\frac{1}{N}\sum_{i=1}^N \nabla K\left(x-X^{(i)}_t\right)\]
where we set $\nabla K(0) = 0$ whenever $\nabla K(0)$ is undefined. The recovery problem we wish to solve is the following.\\

\noindent \textbf{(P)} Let $\pmb{\Xbb} = (\Xbf_\tbf^{(1)}, \dots, \Xbf_\tbf^{(M)})$ be discrete-time data at $L$ timepoints $\tbf := (t_1,\dots,t_L)$ for $M$ i.i.d.\ trials of the process \eqref{dXt} with $K = K^\star$, $V=V^\star$, and $\sigma=\sigma^\star$ and let $\pmb{\Ybb} = \pmb{\Xbb}+\varepsilon$ be a corrupted dataset. For some fixed compact domain $\CalD \subset \Rbb^d$ containing $\supp{\pmb{\Ybb}}$, and finite-dimensional hypothesis spaces\footnote{The set $\CalD-\CalD$ is defined
$\CalD-\CalD =\{x-y\ :\ (x,y)\in \CalD\times \CalD\}$.} $\CalH_K\subset L^2(\CalD-\CalD)$, $\CalH_V\subset L^2(\CalD)$, and $\CalH_\sigma\subset L^2(\CalD)$,
solve 
\[\left(\widehat{K},\widehat{V}, \widehat{\sigma}\right)=\argmin_{K\in \CalH_K,V\in \CalH_V,\sigma\in \CalH_\sigma} \nrm{\nabla K-\nabla K^\star}_{L^2(\CalD-\CalD)}+\nrm{\nabla V-\nabla  V^\star}_{L^2(\CalD)}+\nrm{\sigma-\sigma^\star}_{L^2(\CalD)}.\]
The problem \textbf{(P)} is clearly intractable because we do not have access to $K^\star$, $V^\star$, or $\sigma^\star$, and moreover the interactions between these terms render simultaneous identification of them ill-posed. We consider two cases: (i) $\varepsilon \neq 0$ and $\sigma^\star=0$, corresponding to purely {\it extrinsic noise}, and (ii) $\varepsilon=0$ and $\sigma^\star\neq 0$, corresponding to purely {\it intrinsic noise}. The extrinsic noise case is important for many applications, such as cell tracking, where uncertainty is present in the position measurements. In this case we examine $\varepsilon$ representing i.i.d.\ Gaussian noise with mean zero and variance\footnote{By $\Ibf_d$ we mean the identity in $\Rbb^d$.} $\ep^2\Ibf_d$ added to each particle position in $\pmb{\Xbb}$. In the case of purely intrinsic noise, identification of the diffusivity $\sigma^\star$ is required as well as the deterministic forces on each particle as defined by $K^\star$ and $V^\star$.  A natural next step is to consider the case with both extrinsic and intrinsic noise.  However, this is a topic for future work and thus beyond the scope of this article.

\section{Background}

Interacting particle systems (IPS) such as \eqref{dXt} are used to describe physical and artificial phenomena in a range of fields including astrophysics \cite{warren1992astrophysical,guo2021progress}, molecular dynamics \cite{lelievre2016partial}, cellular biology \cite{sepulveda2013collective,van2015simulating,bi2016motility}, and opinion dynamics \cite{blondel2010continuous}. In many cases the number of particles $N$ is large, with cell migration experiments often tracking $10^3$-$10^6$ cells and simulations in physics (molecular dynamics, particle-in-cell, etc.) requiring $N$ in the range $10^6$-$10^{12}$. Inference of such systems from particle data thus requires efficient means of computing pairwise forces from $\CalO(N^2)$ interactions at each timestep for multiple candidate interaction potentials $K$. Frequently, so-called {\it mean-field} equations at the continuum level are sufficient to describe the evolution of the system, however in many cases (e.g.\ chemotaxis in biology \cite{keller1971model}) only phenomenological mean-field equations are available. Moreover, it is often unclear how many particles $N$ are needed for a mean-field description to suffice. Many fields are now developing machine learning techniques to extract coarse-grained dynamics from high-fidelity simulations (see \cite{gkeka2020machine} for a recent review in molecular dynamics).  In this work we provide a means for inferring governing mean-field equations from particle data assumed to follow the dynamics \eqref{dXt} that is highly efficient for large $N$, and is effective in learning mean-field equations when $N$ is in range $10^3$-$10^5$.\\   

Inference of the drift and diffusion terms for stochastic differential equations (SDEs) is by now a mature field,  with the primary method being maximum-likelihood estimation, which uses Girsanov's theorem together with the Radon-Nykodym derivative to arrive at a log-likelihood function for regression. See \cite{bibby1995martingale,lo1988maximum} for some early works and \cite{bishwal2007parameter} for a textbook on this approach. More recently, sparse regression approaches using the Kramers-Moyal expansion have been developed \cite{boninsegna2018sparse,callaham2021nonlinear,li2021extracting} and the authors of \cite{nardini2021learning} use sparse regression to learn population level ODEs from agent-based modeling simulations. In addition, a neural network-based algorithm was developed in \cite{chen2021solving}. \\

Only in the last few years have significant strides been made towards parameter inference of {\it interacting} particle systems such as \eqref{dXt} from data. Apart from some exceptions, such as a Gaussian process regression algorithm recently developed in \cite{feng2021data}, applications of maximum likelihood theory are by far the most frequently studied. An early but often overlooked work by Kasonga \cite{kasonga1990maximum} extends the maximum-likelihood approach to inference of IPS, assuming full availability of the continuous particle trajectories and the diffusivity $\sigma$. Two decades later, Bishwal \cite{bishwal2011estimation} further extended this approach to discrete particle observations in the specific context of linear particle interactions. In both cases, a sequence of finite-dimensional subspaces is used to approximate the interaction function, and convergence is shown as the dimension of the subspace $J$ and number of particles $N$ both approach infinity. More recently, the maximum likelihood approach has been carried out in \cite{bongini2017inferring,lu2020learning} in the case of radial interactions and in \cite{chen2021maximum} in the case of linear particle interactions and single-trajectory data (i.e.\ one instance of the particle system). The authors of \cite{sharrock2021parameter} recently developed an online maximum likelihood method for inference of IPS, and in \cite{gomes2019parameter} maximum likelihood is applied to parameter estimation in an IPS for pedestrian flow. It should also be noted that parameter estimation for IPS is common in biological sciences, with the most frequently used technique being nonlinear least squares with a cost function comprised of summary statistics \cite{lukeman2010inferring,sepulveda2013collective}.\\ 

Problem \textbf{(P)} is made challenging by the coupled effects of $K$, $V$ and $\sigma$. In each of the previously mentioned algorithms, the assumption is made that $\sigma$ is known and/or that $K$ takes a specific form (radial or linear). In addition, the maximum likelihood-based approach approximates the differential $dX_t^{(i)}$ of particle $i$ using a 1st-order finite difference: $dX_t^{(i)} \approx X_{t+\Delta t}^{(i)} - X_t^{(i)}$, which is especially ill-suited to problems involving extrinsic noise in the particle positions. Our primary goal is to show that the weak-form sparse regression framework allows for identification of the full model $(K,V,\sigma)$, with significantly reduced computational complexity, when $N$ is on the order of several thousands or more. The feasibility of this approach is grounded in the convergence of IPS to associated mean-field equations. The reduction in computational complexity follows from the reduction in evaluation of candidate potentials (as discussed in Section \ref{sec:compcomp}), as well as the convolutional form of the weak-form algorithm.\\

To the best of our knowledge, we present here the first {\it weak-form sparse regression} approach for inference of interacting particle systems. We use a two-step process: the density of particles is approximated using a density kernel $G$ and then the WSINDy algorithm (weak-form sparse identification of nonlinear dynamics) is applied in the PDE setting \cite{messenger2020weak,messenger2020weakpde}. WSINDy is a modified version of the original SINDy algorithm \cite{brunton2016discovering,rudy2017data} where the weak formulation of the dynamics is enforced using a family of test functions that offers reduced computational complexity, high-accuracy recovery in low-noise regimes, and increased robustness to high-noise scenarios. There are two works that are most closely related to the current work. In \cite{supekar2021learning}, the authors learn local hydrodynamic equations from active matter particle systems using the SINDy algorithm in the strong-form PDE setting. In contrast to \cite{supekar2021learning}, our approach learns nonlocal equations using the weak-form, however similarly to \cite{supekar2021learning} we perform model selection and inference of parameters using sparse regression at the continuum level. The weak form provides an advantage because no smoothness is required on the particle density (for requisite smoothness the authors of \cite{supekar2021learning} use a Gaussian kernel, which is more expensive to compute than simple particle binning as done here). In \cite{lang2020learning}, the authors apply the maximum likelihood approach in the continuum setting on the underlying nonlocal Fokker-Planck equation and learn directly the nonlocal PDE using strong-form discretizations of the dynamics. While we similarly use the continuum setting for inference (albiet in weak form), our approach differs from \cite{lang2020learning} in that it is designed for the more realistic setting of discrete-time particle data.

\subsection{Contributions}

The purpose of the present article is to show that the weak form provides an advantage in speed and accuracy compared with existing inference methods for particle systems when the number of particles is sufficiently large (on the order of several thousand or more). The key points of this article include:
\begin{enumerate}[label=(\Roman*)]
\item Formulation of a weak-form sparse recovery algorithm for simultaneous identification of the particle interaction force $K$, local potential $V$, and diffusivity $\sigma$ from discrete-time particle data 
\item $L^1$ convergence of the resulting full-rank least-squares solution as the number of particles $N\to \infty$ and timestep $\Delta t\to 0$
\item Numerical illustration of (i) theoretical convergence rates in $N$ and (ii) robustness to either intrinsic randomness (e.g.\ Brownian motion) or extrinsic randomness (e.g.\ additive measurement noise)
\end{enumerate}

\subsection{Paper Organization}

In Section \ref{sec:meanfield} we review results from mean-field theory used to show convergence of the weak-form method. In Section \ref{sec:algorithm} we introduce the WSINDy algorithm applied to interacting particles, including hyperparameter selection, computational complexity, and convergence of the method under suitable assumptions in the limit of large $N$. Section \ref{sec:examples} contains numerical examples exhibiting the convergence rates of the previous section and examining the robustness of the algorithm to various sources of corruption, and Section \ref{sec:discussion} contains a discussion of extension and future directions.

\section{Review of mean-field theory}\label{sec:meanfield}

Our weak-form approach utilizes that under fairly general assumptions the empirical measure $\mu^N_t$ of the process $\Xbf_t$ defined in \eqref{dXt} converges weakly to $\mu_t$, the distribution of the associated mean-field process $X_t$ defined in \eqref{mckeanvlasovSDE}. Specifically, under suitable assumptions on $V,K,\sigma$ and $\mu_0$, there exists $T>0$ such that for all $t\in[0,T]$, the mean-field limit\footnote{We use the notation $t\to \mu_t$ to denote the evolution of probability measures. Subscripts will not be used to denote differentiation.} 
\[\lim_{N \to \infty}\mu^N_t=\mu_t\]
holds in the weak topology of measures\footnote{Meaning that for all continuous bounded functions $\phi: \Rbb^d\to \Rbb$, $\int_{\Rbb^d}\phi(x)d\mu^N_t(x)\to \int_{\Rbb^d}\phi(x)d\mu_t(x)$.}, where $\mu_t$ is a weak-measure solution to the mean-field dynamics
\begin{equation}\label{fpmeanfield}
\partial_t \mu_t  = \nabla\cdot \left(\left(\nabla K*\mu_t+\nabla V\right)\mu_t\right) + \frac{1}{2}\sum_{i,j=1}^d\frac{\partial^2}{\partial x_i \partial x_j} \left(\sigma \sigma^T \mu_t\right), \quad \mu_0\in \CalP_p(\Rbb^d).
\end{equation}
Equation \ref{fpmeanfield} describes the evolution of the distribution of the McKean-Vlasov process
\begin{equation}\label{mckeanvlasovSDE}
dX_t = \left(-\nabla V\left(X_t\right) -\nabla K*\mu_t\left(X_t\right)\right) dt + \sigma(X_t)\,dB_t.
\end{equation}
This implies that as $N\to \infty$, an initially correlated particle system driven by pairwise interaction becomes uncorrelated and only interacts with its mean field $\mu_t$. In particular, the following theorem summarizes several mean-field results taken from the review article \cite{jabin2017mean} with proofs in \cite{sznitman1991topics,meleard1996asymptotic}. (Note that for a function $f:\Rbb^d\to Y$, where $Y$ is a metric space with metric $\rho$, we define $\text{Lip}_S(f)$ by 
\[\text{Lip}_S(f):=\sup_{\{x,y\in S\}} \frac{\rho(f(x),f(y))}{\left\vert x- y\right\vert}\]
with $\text{Lip}(f):=\text{Lip}_{\Rbb^d}(f)$. Throughout we use $|\cdot|$ to denote the Euclidean norm.)
\begin{thm}\label{meanfieldthm1} 
Assume that $K$ is globally Lipschitz, $V=0$, and $\sigma(x)=\sigma=\text{const.}$ In addition assume that $\mu_0\in \CalP_2(\Rbb^d)$. Then for any $T>0$, for all $t\leq T$ it holds that
\begin{enumerate}[label=(\roman*)]
\item There exists a unique solution $(X_t, \mu_t)$ where $X_t$ is a strong solution to \eqref{mckeanvlasovSDE} and $\mu_t$ is a weak-measure solution to \eqref{fpmeanfield}. 
\item For any $\phi\in C^1_b(\Rbb^d)$,
\begin{equation}\label{propchaos}
\Ebb\left\vert \frac{1}{N}\sum_{i=1}^N \phi(X_i(t))-\int_{\Rbb^d}\phi(x)d\mu_t(x)\right\vert\leq \frac{C}{\sqrt{N}}
\end{equation}
with $C$ depending on $\text{Lip}(\nabla K)$, $\nrm{\phi}_{C^1}$ and $T$.
\item For any $k\in \Nbb$, $a.e.$-$t< T$, the $k$-particle marginal
\[\rho_t^{(k)}(x_1,\dots,x_k):=\int_{\Rbb^{d(N-k)}} F_t^N(x_1,\dots,x_k,x_{k+1},\dots,x_N)\,dx_{k+1}\cdots dx_N\]
converges weakly to $\mu_t^{\otimes k}$ as $N\to \infty$, where $F^N_t\in\CalP(\Rbb^{Nd})$ is the distribution of $\Xbf_t$.
\end{enumerate}
\end{thm}

Theorem \ref{meanfieldthm1} immediately extends to the case of $V$ and $\sigma$ both globally Lipschitz and has been extended to $K$ locally-Lipschitz in \cite{bolley2011stochastic}, $K$ with Coulomb-type singularity at the origin in \cite{boers2016mean}, and domains with boundaries in \cite{fetecau2019propagation,fetecau2018zero}. Analysis of the model \eqref{fpmeanfield} continues to evolve in various contexts, including with analysis of equilibria \cite{messenger2020equilibria,fetecau2017swarm,carrillo2019existence} and connections to deep learning \cite{araujo2019mean}.

\subsection{Weak form}

Despite the $\CalO(N^{-1/2})$ convergence of the empirical measure in Theorem \ref{meanfieldthm1}, it is unclear at what particle number $N$ the mean-field equations become a suitable framework for inference using particle data, due to the complex variance structure at any finite $N$. A key piece of the present work is to show that the weak form of the mean-field equations does indeed provide a suitable setting when $N$ is at least several thousand. Moreover, since in many cases \eqref{fpmeanfield} can only be understood in a weak sense, the weak form is the natural framework for identification. We say that $\mu_t$ is a weak solution to \eqref{fpmeanfield} if for any $\psi\in C^2(\Rbb^d\times(0,T))$ compactly supported it holds that 
{\small
\begin{equation}\label{weakform}
\int_0^T\int_{\Rbb^d} \partial_t\psi(x,t)\,d\mu_t(x)dt= \int_0^T\int_{\Rbb^d} \left(\nabla\psi(x,t)\cdot\left(\nabla K*\mu_t(x)+\nabla V(x)\right) - \frac{1}{2}\text{Tr}\left(\nabla^2\psi(x,t)\sigma(x)\sigma^T(x)\right)\right)d\mu_t(x)dt,
\end{equation}}
where $\nabla^2\psi$ denotes the Hessian of $\psi$ and $\text{Tr}(\Abf)$ is the trace of the matrix $\Abf$. Our method requires discretizing \eqref{weakform} for all $\psi\in \Psi$ where $\Psi = (\psi_1,\dots,\psi_n)$ is a suitable test function basis, and approximating the mean field density $\mu_t$ with a discrete density $U_t$ constructed from particle data at time $t$. We then find $K, V$, and $\sigma$ within specified finite-dimensional function spaces. 

\section{Algorithm}\label{sec:algorithm}

We propose the general algorithm \ref{alg1} for discovery of mean-field equations from particle data. The inputs are a discrete-time sample $\pmb{\Ybb}$ containing $M$ experiments each with $N$ particle positions over $L$ timepoints $\tbf=(t_1,\dots,t_L)$, and the following hyperparameters are defined by the user: (i) a kernel $G$ used to map the empirical measure $\mu^N_t$ to an approximate density $U_t$, (ii) a spatial grid $\Cbf$ over which to evaluate the approximate density $\Ubf_t=U_t(\Cbf)$,  (iii) a library of trial functions $\Lbb = \{\Lbb_K,\Lbb_V,\Lbb_\sigma\}=\{(K_j)_{j=1}^{J_K},(V_j)_{j=1}^{J_V},(\sigma_j)_{j=1}^{J_\sigma}\}$, (iv) a basis of test functions $\Psi=(\psi_k)_{k=1}^n$, and (v) a quadrature rule over $(\Cbf,\tbf)$ denoted by an inner product $\lan\cdot,\cdot\ran$, and (vi) sparsity factors $\pmb{\lambda}$ for the modified sequential thresholding algorithm (MSTLS) reviewed below. We discuss the choices of these hyperparameters in Section \ref{sec:hype}, the computational complexity of the algorithm in Section \ref{sec:compcomp}, and convergence of the algorithm in Section \ref{sec:conv}.

\begin{algorithm}
	\caption{ WSINDy for Particle Systems \\ $(\what, \, \hat{\lambda}) =$ \textbf{WSINDy}$(\pmb{\Ybb},\tbf\,;\ G,\,\Cbf,\, \Lbb,\,\Psi,\,\lan \cdot,\cdot\ran,\,\pmb{\lambda})$}
	\label{alg:wsindypde}
	\begin{algorithmic}[1]
		\FOR{$\ell=1:L$}
			\FOR{$m=1:M$}
				\STATE $\Ubf_\ell^{(m)} = \int_{\Rbb^d} G(\Cbf,y)d\mu^N_{t_\ell}(y)$ where $\mu^N_{t_\ell}$ is the empirical measure for $\Ybf_{t_\ell}^{(m)}$ 
		\ENDFOR
		\STATE $\Ubf_\ell = \frac{1}{M}\sum_{m=1}^M \Ubf_\ell^{(m)}$
		\ENDFOR
		\STATE
		\FOR{$j=1:J_K$}
			\FOR{$k=1:n$}
				\STATE $\Gbf_{kj}^K =  \lan \nabla \psi_k, \Ubf \nabla K_j*\Ubf \ran$
			\ENDFOR
		\ENDFOR
		\STATE
		\FOR{$j=1:J_V$}
			\FOR{$k=1:n$}
				\STATE $\Gbf_{kj}^V =  \lan \nabla \psi_k, \Ubf \nabla V_j \ran$
			\ENDFOR
		\ENDFOR
		\STATE
		\FOR{$j=1:J_\sigma$}
			\FOR{$k=1:n$}
				\STATE $\Gbf_{kj}^\sigma =  \frac{1}{2}\sum_{p,q=1}^d \lan \partial_{x_px_q}\psi_k, (\sigma_j\sigma_j^T)_{pq} \Ubf \ran$
			\ENDFOR
		\ENDFOR
		\STATE $\Gbf = [\Gbf^K\ \Gbf^V\ \Gbf^\sigma]$
		\STATE
		\FOR{$k=1:n$}
			\STATE $\bbf_k = \lan \partial_t \psi_k, \Ubf\ran$
		\ENDFOR
		\STATE
		\STATE $(\widehat{\wbf},\widehat{\lambda}) = \text{MSTLS}(\Gbf,\bbf;\, \pmb{\lambda})$
\end{algorithmic}\label{alg1}
\end{algorithm}

\subsection{Hyperparameter Selection}\label{sec:hype}
\subsubsection{Quadrature}

We assume that the set of gridpoints $\Cbf$ in Algorithm \ref{alg1} is chosen from some compact domain $\CalD\subset \Rbb^d$ containing $\supp{\pmb{\Ybb}}$. The choice of $\Cbf$ (and $\CalD$) must be chosen in conjunction with the quadrature scheme, which includes integration in time using the given timepoints $\tbf$ as well as space. For completeness, the inner products in lines 10, 16, 22 and 27 of Algorithm \ref{alg1} are defined in the continuous setting by 
\[\lan f,g\ran = \int_0^T\int_{\CalD} f(x,t)g(x,t)dxdt,\]
and the convolution in line 10 is defined by
\[\nabla K_j * U_t(x) = \int_{\CalD} \nabla K_j(x-y) U_t(y)dy.\]
In the present work we adopt the scheme used in the application of WSINDy for local PDEs \cite{messenger2020weakpde}, which includes the trapezoidal rule in space and time with test functions $\psi$ compactly supported in $\CalD\times (0,T)$. We take $\CalD$ to be an equally-spaced rectangular grid enclosing $\supp{\pmb{\Ybb}}$ in order to efficiently evaluate convolution terms. In what follows we denote by $\lan\cdot, \cdot\ran$ the continuous inner product, $\lan \cdot,\cdot\ran_h$ the inner product over $\CalD\times[0,T]$ evaluated using the composite trapezoidal rule in space with meshwidth $h$ and Lebesgue integration in time, and by $\lan \cdot,\cdot\ran_{h,\Delta t}$ the trapezoidal rule in both space and time, with meshwidth $h$ in space and $\Delta t$ in time. With some abuse of notation, $f*g$ will denote the convolution of $f$ and $g$, understood to be discrete or continuous by the context. Note also that we denote $\mu^N$, $\mu$ and $U$ the measures over $\Rbb^d\times [0,T]$ defined by $\mu^N_t\lambda_{[0,T]}$, $\mu_t\lambda_{[0,T]}$ and $U_t\lambda_{[0,T]}$, respectively, where $ \lambda_{[0,T]} $ is the Lebesgue measure on $[0,T]$.


\subsubsection{Density Kernel}\label{sec:hist}

Having chosen the domain $\CalD\subset \Rbb^d$ containing the particle data $\pmb{\Ybb}$, let $P^h = \left\{ B_k\right\}_k$ be a partition of $\CalD$ ($\cup_k B_k = \CalD$) with $h$ indicating the size of the atoms $B_k$. For the remainder of the paper we take $B_k$ to be hypercubes of equal side length $h$ in order to minimize computation time for integration, although this is by no means necessary. For particle positions $\Xbf_t$, we define the histogram\footnote{The indicator function is defined $\ind{A}(x):=\begin{cases} 1, & x\in A\\0, & x\notin A\end{cases}$.}
\begin{equation}\label{histU}
U_t = \sum_k \frac{1}{|B_k|}\ind{B_k}(x) \left(\frac{1}{N}\sum_i \ind{B_k}(X^{(i)}_t)\right) = \int_\CalD G(x,y)d\mu^N_t(y).
\end{equation}
Here the {\it density kernel} is defined 
\[G(x,y) = \sum_k \frac{1}{|B_k|}\ind{B_k}(x)\ind{B_k}(y),\]
and in this setting the corresponding spatial grid $\Cbf = (\cbf_k)_k$ is the set of center-points of the boxes $B_k$, from which we define the discrete histogram data $\Ubf_t = U_t(\Cbf)$. The discrete histogram $\Ubf_t$ then serves as an approximation to the mean-field distribution $\mu_t$.\\

Pointwise estimation of densities from samples of particles usually requires large numbers of particles to achieve reasonably low variance, and in general the variance grows inversely proportional to the bin width $h$. One benefit of the weak form is that integrating against a histogram $U$ does not suffer from the same increase in variance with small $h$. In particular,
\begin{lemm}\label{lemm:histint}
Let $\Ybf = (Y^{(1)},\dots,Y^{(N)})$ be i.i.d.\ samples from $\mu\in \CalP(\Rbb^d)$ with associated empirical measure $\mu^N$ and let $U$ be the histogram computed with kernel $G$ using \eqref{histU} with $n$ bins of equal sidelength $h$. Then for any $\psi$ in $C^1$ compactly supported in $\CalD$, we have the root-mean-squared error
\[\left(\Ebb\left[\left(\lan \psi, U\ran_{\CalT} - \lan \psi, \mu\ran\right)^2\right]\right)^{1/2}\leq  \nrm{\psi}_{C^1}\left(\sqrt{2}^{d-3} h + N^{-1/2}\right).\]
\end{lemm}
\begin{proof}
Using the compact support of $\psi$, we have 
\[\lan \psi, U\ran_{\CalT} = \lan \psi, \int_{\Rbb^d} G(\cdot, y)d\mu^N\ran_h = \lan \psi^\Cbf, \mu^N\ran = \frac{1}{N}\sum_{i=1}^N \psi^\Cbf(Y^{(i)})\]
where 
\begin{equation}\label{psiC}
\psi^\Cbf(x)=\sum_{k=1}^K \psi(c_k) \ind{B_k}(x)
\end{equation}
is the midpoint approximation of $\psi$. We then have the squared bias
\begin{align*}
\text{bias}(\lan \psi, U\ran_{\CalT})^2&= \left(\Ebb\left[\lan \psi, U\ran_{\CalT}\right] -\int_\CalD \psi(x)d\mu(x)\right)^2\\
 &= \left(\int_\CalD \psi^\Cbf(x) d\mu(x) -\int_\CalD \psi(x)d\mu(x)\right)^2\\
&\leq \nrm{\nabla \psi}^2_\infty 2^{d-3} h^2
\end{align*}
and the variance, using the fact that $Y^{(i)}$ are independent,
\[\text{Var}\left(\lan \psi, U\ran_{\CalT}\right) =\frac{1}{N}\Ebb_{X\sim \mu}\Big[\psi^\Cbf(X)\left(\psi^\Cbf(X) - \Ebb_{Z\sim \mu}\left[\psi^\Cbf(Z)\right]\right)\Big] \leq \frac{1}{N}\nrm{\psi}_\infty^2.\]
The result follows since 
\[\Ebb\left[\left(\lan \psi, U\ran_{\CalT} - \lan \psi, \mu\ran\right)^2\right] = \text{bias}(\lan \psi, U\ran_{\CalT})^2+\text{Var}\left(\lan \psi, U\ran_{\CalT}\right).\]
\end{proof}
The previous lemma in particular shows that small bin-width $h$ does not negatively impact $\lan \psi, U\ran_h$ as an estimator of $\lan \psi, \mu\ran$, which is in contrast to $U(x)$ as a pointwise estimator of $\mu(x)$. For example, if we assume that $\Ybf$ is sampled from a $C^1$ density $\mu$, it is well known that the mean-square optimal bin width is $h=\CalO(N^{-1/3})$ \cite{freedman1981histogram}. Summarizing this result, elementary computation reveals the pointwise bias for $x\in B_k$, 
\[\text{bias}(U(x))=\Ebb\left[U(x)\right] - \mu(x) = \frac{\mu(B_k)}{|B_k|} -\mu(x) := \mu(\xi)-\mu(x)\]
for some $\xi\in B_k$. Letting $L_k = \max_{x\in B_k}|\nabla \mu(x)|$, we have
\[\text{bias}(U(x))^2 \leq L_k^2 2^{d-1}h^2.\]
For the variance we get
\[\text{Var}\left(U(x)\right) = \frac{1}{N}\frac{\mu(B_k)(1-\mu(B_k))}{|B_k|^2} = \frac{\mu(\xi)}{N}\left(1-\mu(B_k)\right)\frac{1}{\sqrt{2}^{d-1}h},\] 
and hence a bound for the mean-squared error
\[\Ebb\left[\left(U(x)-\mu(x)\right)^2\right]\leq L_k^2 2^{d-1}h^2+ \frac{\mu(\xi)}{N\sqrt{2}^{d-1}}h^{-1}.\]
Minimizing the bound over $h$ we find an approximately optimal box width 
\[h^* = \left(\frac{\rho(\xi)}{2^{\frac{3d-1}{2}}L_k^2}\right)^{1/3}N^{-1/3} = \CalO(N^{-1/3}),\]
which provides an overall pointwise root-mean-squared error of $\CalO(N^{-1/3})$. Hence, not only does the weak-form remove the inverse $h$ dependence in the variance, but fewer particles are needed to accurately approximate integrals of the density $\mu$.


\subsubsection{Trial Function Library}\label{subsec:lib}

The general Algorithm \ref{alg1} does not impose a radial structure for the interaction potential $K$, nor does it assume any prior knowledge that the particle system is in fact interacting. In the examples below we use monomial and/or trigonometric libraries for $K$, $V$, and $\sigma$ to show that sparse regression is effective in selecting the correct interaction terms from a library also containing local drift terms. (Details of the libraries used in examples can be found in Appendix \ref{app:specs}).

\subsubsection{Test Function Basis}

For the test functions $(\psi_k)_{k\in[n]}$ we use the same approach as the PDE setting \cite{messenger2020weakpde}, namely we fix a {\it reference test function} $\psi$ and set 
\[\psi_k(x,t) = \psi(\xbf_k-x,t_k-t)\]
where $\CalQ:=\{(\xbf,t_k)\}_{k\in[n]}$ is a fixed set of {\it query points}. This, together with a separable representation
\[\psi(x,t) = \phi_1(x_1)\cdots\phi_d(x_d)\phi_{d+1}(t),\]
enables construction of the linear system $\Gbf$, $\bbf$ using the FFT. We choose $\phi_j$, $1\leq j\leq d+1$, of the form 
\begin{equation}\label{testfcn}
\phi_{m,p}(v;\, \Delta) := \max\left\{1-\left(\frac{v}{m\Delta}\right)^2,0\right\}^{p}
\end{equation}
where $m$ is the integer {\it support parameter}  such that $\phi_{m,p}$ is supported on $2m+1$ points of spacing $\Delta\in\{h, \Delta t\}$ and $p\geq 1$ is the {\it degree} of $\phi_{m,p}$. Using the convergence analysis below, we need $\phi_j$ to be three times continuously differentiable, yet with Lipschitz constant as low as possible, hence for $\phi_{d+1}$ (along the time axis) we pick $p=p_t=3$, while for all spatial dimensions we set $p=p_x=5$. We choose the support parameters $m_t$ and $m_x$ using the changepoint algorithm in \cite[Appendix A]{messenger2020weakpde}, enforcing that each $\phi_j$ is supported on a minimum of 6 gridpoints and a maximum of half of the width of the domain along the respective coordinate. For $Q$, we sample points uniformly within $\CalC$ of frequency $s_x$ or $s_t$ depending on if the coordinate in spatial or temporal. The resulting values $m_x,m_t,p_x,p_t,s_x,s_t$ for each example can be found in Appendix \ref{app:specs}.

\subsubsection{Sparsity Regression}

As in \cite{messenger2020weakpde}, we enforce sparsity using a {\it modified} sequential thresholding least-squares algorithm. With\footnote{The Moore-Penrose inverse $\Abf^\dagger$ is defined for a rank-$r$ matrix $\Abf$ using the reduced SVD $\Abf=U_r\Sigma_r V^*_r$ as $\Abf^\dagger:=V_r\Sigma^{-1}_rU_r^*$. The subscript $r$ denotes restriction to the first $r$ columns.} $\wbf^0 = \Gbf^\dagger \bbf$, we define
\begin{equation}\label{MSTLS1}
\text{MSTLS}(\Gbf,\bbf; \lambda\,)\qquad \begin{dcases} \hspace{0.43cm}\CalI^\ell = \{1\leq i\leq SJ\ :\ L^\lambda_i\leq|\wbf^\ell_i|\leq U^\lambda_i\} \\
\wbf^{\ell+1} = \argmin_{\supp{\wbf}\subset \CalI^\ell} \nrm{ \Gbf  \wbf-\bbf}_2^2.\end{dcases}
\end{equation}
where the bounds are defined
\begin{equation}\label{MSTLSbnds} 
\begin{dcases} L_i^\lambda =  \lambda\max\left\{1,\ \frac{\nrm{\bbf}}{\nrm{\Gbf_i}}\right\}\\
U_i^\lambda =  \frac{1}{\lambda}\min\left\{1,\ \frac{\nrm{\bbf}}{\nrm{\Gbf_i}}\right\}\end{dcases}, \qquad 1\leq i\leq SJ.
\end{equation}
We then select the sparsity threshold $\widehat{\lambda}$ as the smallest minimizer of the cost function  \begin{equation}\label{lossfcn}
\CalL(\lambda) = \frac{\nrm{\Gbf(\wbf^\lambda-\wbf^0)}_2}{\nrm{\Gbf\wbf^0}_2}+\frac{\nrm{\wbf^\lambda}_0}{SJ}
\end{equation}
over all $\lambda$ in a specified finite set $\pmb{\lambda}$, where $\wbf^\lambda:=\text{MSTLS}(\Gbf,\bbf; \lambda\,)$. We set the final model coefficients to $\what:=\wbf^{\widehat{\lambda}}$. The bounds \eqref{MSTLSbnds} enforce a quasi-dominant balance rule, such that $\nrm{\wbf_i\Gbf_i}_2$ is within $-\log_{10}(\lambda)$ orders of magnitude from $\nrm{\bbf}_2$ and $|\wbf_i|$ is within $-\log_{10}(\lambda)$ orders of magnitude from $1$ (the coefficient of time derivative $\partial_t\mu_t$). Minimizers of the cost function $\CalL$ then equally weight the accuracy and sparsity of $\wbf^{\widehat{\lambda}}$. By choosing $\widehat{\lambda}$ to be the smallest minimizer of $\CalL$ over $\pmb{\lambda}$, we identify the thresholds $\lambda\in \pmb{\lambda}$ such that $\lambda<\widehat{\lambda}$ as those resulting in an overfit model.

\subsection{Computational Complexity}\label{sec:compcomp}

To compute convolutions against each $\nabla K_j$, we evaluate $\partial_{x_i} K_j$, $1\leq i\leq d$, at the grid $\Cbf-\Cbf$ defined by 
\[\Cbf-\Cbf:=\{x\in \Rbb^d\ :\ x =(i_1h,\dots, i_dh), \quad -m_\ell\leq i_\ell\leq m_\ell\}\]
where $h$ is the spacing of $\Cbf$ and $m_\ell$, $1\leq \ell\leq d$, is the number of points in $\Cbf$ along the $\ell$th spatial dimension. Then $\Cbf-\Cbf$ discretizes the set 
\[\CalD-\CalD:=\{x-y\in \Rbb^d\ :\ (x,y)\in \CalD\times \CalD\}\]
which contains all observed interparticle distances. (In words, to form $\Cbf-\Cbf$ we shift $\Cbf$ to lie in the positive orthant $\{x\in \Rbb^d\ :\ x_\ell\geq 0, \quad 1\leq\ell\leq d\}$, and then reflect $\Cbf$ through each coordinate plane $x_\ell=0$, $1\leq\ell\leq d$.) In this way $\partial_{x_i}K_j$ is evaluated at $2^d|\Cbf|$ points, where $|\Cbf|$ is the number of points in the grid $\Cbf$. Define $\partial_{x_i}\Kbf_j := \partial_{x_i}K_j(\Cbf-\Cbf)$. Since $\Cbf$ is equally spaced, we use the $d$-dimensional FFT to compute the convolutions
\[ \partial_{x_i}\Kbf_j * \Ubf_t \approx \partial_{x_i}K_j*U_t(\Cbf), \quad t\in \tbf,\]
where only entries corresponding to interactions within $\Cbf$ need to be retained. For $d=1$ this amounts to $\CalO(|\Cbf|\log|\Cbf|)$ flops per timestep. For $d=2$ and higher dimensions, the $d$-dimensional FFT is considerably slower unless one of the arrays is separable. Trial interaction potentials $K_j$ can be chosen to be a sum of separable functions,
\[K(x) = \sum_{q=1}^Q k_{1,q}(x_1)\cdots k_{d,q}(x_d),\]
in which case only a series of one-dimensional FFTs are needed, and again the cost is $\CalO(|\Cbf|\log|\Cbf|)$ per timestep. When $K$ is not separable, we propose using a low-rank approximation 
\[\partial_{x_i}\Kbf_j = \sum_{q=1}^Q\sigma_{q} \kbf_{1,q}\otimes \cdots \otimes \kbf_{d,q}\]
to exploit the efficiency of FFT in one dimension. For $d=2$, this is accomplished using the truncated SVD, while for higher dimensions there does not exist a unique {\it best} rank-$Q$ tensor approximation, although several efficient algorithms are available to compute a sufficiently accurate decomposition \cite{malik2018low,sun2020low,jang2020d} (and the field of fast tensor decompositions is advancing rapidly). In the examples below we consider only $d=1$ and $d=2$, and leave extension to higher dimensions to future work.\\

Using low-rank approximations, the mean-field approach provides a significant reduction in computational complexity compared to direct evaluations of particle trajectories when $N$ is sufficiently large. A particle-level computation of the nonlocal force in weak-form requires evaluating terms of the form
\[\sum_{\ell=1}^L\left(\frac{1}{N^2}\sum_{i=1}^N\sum_{j=1}^N\partial_x\psi(X^{(i)}_{t_\ell},t_\ell)\partial_x K(X^{(i)}_{t_\ell}- X^{(j)}_{t_\ell})\right)\Delta t.\]
For a single candidate interaction potential $K$, a collection of $J$ test functions $\psi$, and $M$ experiments, this amounts to $JMNL+MLN^2$ function evaluations in $\Rbb^d$ and $\CalO(MJLN^2)$ flops. If we use the proposed method, exploiting the convolutional structure of integration against a separable reference test function $\psi$ and a rank $Q$ approximation of $\partial_x \Kbf$, we instead evaluate  
\[\partial_x\psi *\left(U(\partial_x K*U)\right)\]
using $2^d|\Cbf|$ evaluations of $\partial_x K$, reused at each of the $L$ timepoints, and $\CalO(L|\Cbf|Q\log(|\Cbf|))$ flops\footnote{Neglecting the cost of computing the histogram $\Ubf$ and evaluating $\psi(\Cbf)$, amounting to an additional $\CalO(NML+|\Cbf|)$ flops, as these terms are reused in each column of $\Gbf$ and $\bbf$}. Figure \ref{speedup} provides a visualization of the reduction in function evaluations for $L=100$ timepoints and $M=10$ experiments over a range of $N$ and $|\Cbf|^{1/d}$ (points along each spatial dimension when $|\Cbf|$ is a hypercube) in $d=2$ and $d=3$ spatial dimensions. Table \ref{specs} in Appendix \ref{app:specs} lists walltimes for the examples below, showing that with $N=64,000$ particles the full algorithm implemented in MATLAB runs in under 10 seconds with all computations in serial on a laptop with an AMD Ryzen 7 pro 4750u processor and 38.4 GB of RAM. The dependence on $N$ is only through the $\CalO(N)$ computation of the histrogram, hence this approach may find applications in physical coarse-graining (e.g.\ of molecular dynamics or plasma simulations).\\

\begin{figure}
\centering
\begin{tabular}{cc}
\hspace{-0.5cm}	\includegraphics[trim={0 5 25 0},clip,width=0.4\textwidth]{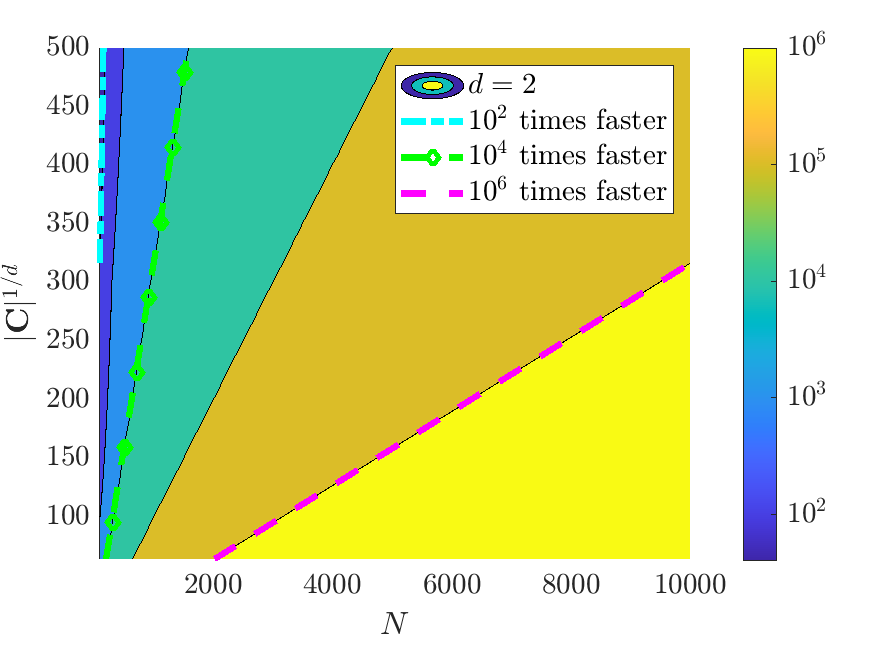}&
\hspace{-0.5cm}	\includegraphics[trim={0 5 25 0},clip,width=0.4\textwidth]{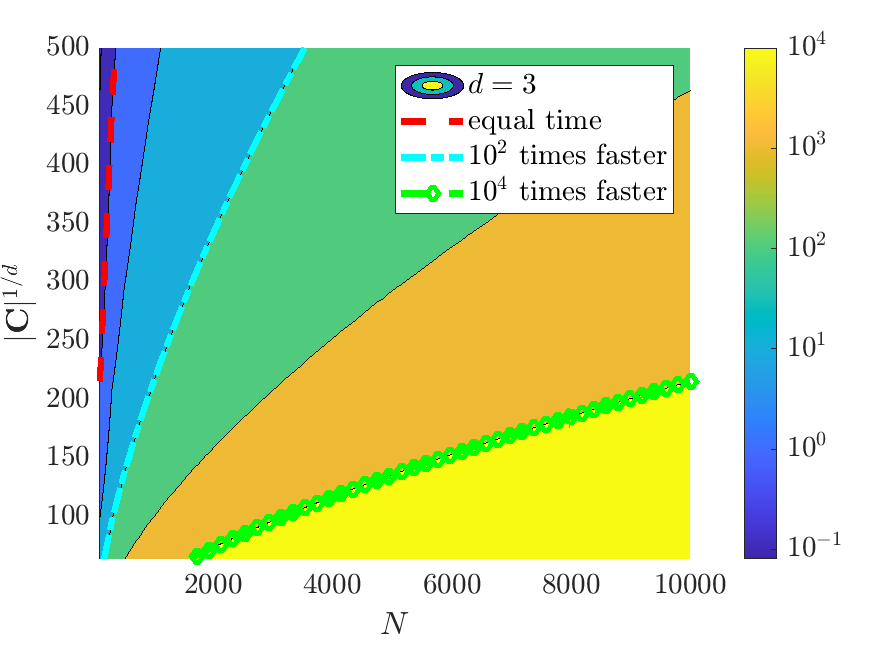}
\end{tabular}
\caption{Factor by which the mean-field evaluation of interaction forces using histograms reduces total function evaluations as a function of particle number $N$ and average gridpoints per coordinate $|\Cbf|^{1/d}$ for data with $M=10$ experiments each with $L=100$ timepoints. For example, with $d=2$ spatial dimensions (left) and $N>2000$ particles, the number of function evaluations is reduced by at least a factor of $10^4$.}
\label{speedup}
\end{figure}

\subsection{Convergence}\label{sec:conv}

We now show that the estimators $\widehat{K}$, $\widehat{V}$, and $\widehat{\sigma}$ of the weak-form method converge with a rate $\CalO(h+N^{-1/2}+\Delta t^\eta)$ when ordinary least squares is used (i.e.\ $\pmb{\lambda}=0$) and only $M=1$ experiment is available. Here $\eta>0$ is the H\"older exponent of the sample paths of the process $\Xbf_t$. We assume that $\CalD$, $\Cbf$, $G$, $P^h$ and the resulting histogram $\Ubf=(\Ubf_t)_{t\leq T}$ are as in Section \ref{sec:hist}. We make the following assumptions on the true model and resulting linear system throughout this section.

\begin{assump}\label{assump1}
Let $p\geq 1$ be fixed. 
\begin{enumerate}[label=(\Roman*)]
\item For each $N\geq 2$, $\Xbf_t=(X^{(1)}_t,\dots,X^{(N)}_t)$ is a strong solution to \eqref{dXt} for $t\in[0,T]$, and for some $\eta>0$ the sample paths $t\to X^{(i)}_t(\omega)$ are almost-surely $\eta$-H\"older continuous, i.e.\
\[|X^{(i)}_t(\omega)-X^{(i)}_s(\omega)|\leq C_\eta|t-s|^\eta, \quad \forall 0\leq s\leq t\leq T,\quad  \forall i\in [N],\quad \text{for a.e. }\omega\in \Omega.\]
\item The initial particle distribution $\mu_0$ satisfies the moment bound 
\[\int_{\Rbb^d}|x|^pd\mu_0(x):= M_p<\infty.\]
\item $\nabla K^\star$ and $\nabla V^\star$ satisfy for some $C_p>0$ the growth bound:
\[|\nabla V^\star(x) -\nabla V^\star(y)|+|\nabla K^\star(x) -\nabla K^\star(y)|\leq C_p|x-y|(1+\max\{|x|,|y|\}^{p-1}), \quad x,y\in \Rbb^d\]
\item For the same constant $C_p>0$, it holds that\footnote{For $\Abf\in \Rbb^{d\times d}$ the Frobenius norm is defined $\nrm{\Abf}_F=\sqrt{\text{Tr}(\Abf^T\Abf)}$}
\[\nrm{\sigma^\star(x)-\sigma^\star(y)}_F\leq C_p|x-y|^{1/2}(1+\max\{|x|,|y|\}^{p/2-1/2}), \quad x,y\in \Rbb^d\]
\item The test functions $(\psi_k)_{k\in[n]}\subset C^3(\Rbb^d\times(0,T))$ are compactly supported and together with the library $\Lbb$ are such that $\Gbf$ has full column rank. Moreover $\nrm{\Gbf^\dagger}_1\leq C_\Gbf$ almost surely, where $\nrm{\Gbf^\dagger}_1$ is the induced matrix 1-norm of $\Gbf^\dagger$.
\item The true functions $K^\star$, $V^\star$ and $\sigma^\star$ are in the span of $\Lbb$.
\end{enumerate}
\hrule
\end{assump}
\begin{rmrk} 
Some consequences of Assumption \ref{assump1} are the following: $\eta$-H\"older continuous sample paths implies that for each $t\in[0,T]$,
\[\int_{\Rbb^d}|x|^pd\mu^N_t=\frac{1}{N}\sum_{i=1}^N|X^{(i)}_t|^p\leq \frac{2^p}{N}\sum_{i=1}^N|X^{(i)}_0|^p+C_\eta 2^pt^{p\eta}.\]
Together with the $p$th moment bound on $\mu_0$, this implies 
\begin{equation}\label{momentbound}
\Ebb\left[\sup_{t\leq T}\int_{\Rbb^d}|x|^pd\mu^N_t\right]\leq 2^p(M_p+C_\eta T^{p\eta}),
\end{equation}
independent of $N$. The growth bounds on $\nabla K^\star$, $\nabla V^\star$ and $\sigma^\star$ imply that for some $C>0$,
\begin{equation}\label{growthbound}
|\nabla K^\star(x)|+|\nabla V^\star(x)|+\nrm{\sigma^\star(x)(\sigma^\star(x))^T}_F\leq C(1+|x|^p),
\end{equation}
where $\nrm{\cdot}_F$ is the Frobenius norm. 
\end{rmrk}
 We will now define some notation and prove some lemmas. Define the weak-form operator
\begin{equation}
\scL(\rho,\psi,\lan\cdot, \cdot\ran):= \lan \partial_t\psi - \nabla \psi \cdot \nabla K^\star*\rho - \nabla \psi \cdot\nabla V^\star + \frac{1}{2}\text{Tr}\left(\nabla^2 \psi \sigma^\star(\sigma^\star)^T\right), \rho\ran, 
\end{equation}
where $\rho=(\rho_t)_{t\leq T}$ is a curve in $\CalP_p(\Rbb^d)$, $\psi$ is a $C^2$ function compactly supported over $\Rbb^d\times (0,T)$, and $\lan \cdot, \cdot\ran$ is an inner product over $\Rbb^d\times (0,T)$. If $\rho = (\mu_t)_{t\leq T}$ is a weak solution to \eqref{fpmeanfield} and $\lan\cdot,\cdot\ran$ is the $L^2(\Rbb^d)$ inner product then $\scL(\rho,\psi,\lan\cdot, \cdot\ran)=0$. If instead $\rho=(\mu^N_t)_{t\leq T}$, then by It\^o's formula $\scL(\rho,\psi,\lan\cdot, \cdot\ran)$ takes the form of an It\^o integral, and we have the following:
\begin{lemm}\label{itobound}
There exists a constant $C>0$ independent of $N$, such that
\[\Ebb\left[\left\vert\scL(\mu^N,\psi, \lan \cdot,\cdot\ran)\right\vert\right]\leq \frac{C}{\sqrt{N}}.\]
\end{lemm}
\begin{proof}
Applying It\^o's formula to the process $\frac{1}{N}\sum_{i=1}^N\psi(X^{(i)},t)$, we get that
\[\scL(\mu^N,\psi, \lan \cdot,\cdot\ran) = \frac{1}{N}\sum_{i=1}^N\int_0^T\nabla \psi(X^{(i)}_t,t)^T\sigma^\star(X^{(i)}_t) dB^{(i)}_t.\]
Note that each integral on the right-hand side is a local martingale, since \eqref{growthbound} ensures boundedness of $\nabla \psi(x,t)^T\sigma^\star(x)$ over any compact set in $\Rbb^d$, hence has mean zero. By independence of the Brownian motions $B^{(i)}_t$, exchangeability of $X^{(i)}_t$, the moment bound \eqref{momentbound}, and the growth bounds on $\sigma$, the It\^o isometry gives us
\begin{align*}
\Ebb\left[\scL(\mu^N,\psi, \lan \cdot,\cdot\ran)^2 \right] &= \frac{1}{N}\int_0^T\Ebb_{X\sim \rho^{(1)}}\left[\left\vert\nabla \psi(X,t)^T\sigma^\star(X)\right\vert^2\right] dt \\
&=\frac{1}{N}\int_0^T\Ebb\left[\int_{\Rbb^d} \left\vert\nabla \psi(x,t)^T\sigma^\star(x)\right\vert^2 d\mu^N_t(x)\right] dt \\
&\leq \frac{C'}{N}\nrm{\nabla\psi}^2_{2,\infty}\int_0^T\Ebb\left[1+\int_{\Rbb^d}|x|^pd\mu^N_t(x)\right]dt\\
&\leq CN^{-1}
\end{align*}
where $C$ depends on $M_p$, $C_p$, $T$ and $\psi$. 
Above, $\rho^{(1)}$ is the $X^{(1)}_t$-marginal of the process $\Xbf_t\in \Rbb^{dN}$, and $\nrm{f(x)}_{p,q}$ for vector-valued functions $f$ denotes the $L^q$ norm over $x$ of the $\ell^p$ norm of $f(x)$. The result follows from Jensen's inequality.
\end{proof}

With the following lemma, we can relate the histogram $U$ to the empirical measure $\mu^N$ through $\scL$ using the inner product $\lan\cdot,\cdot\ran_h$ defined by trapezoidal-rule integration in space and continuous integration in time. 


\begin{lemm}\label{diffcont}
For $C$ independent of $N$ and $h$, it holds that
\[\Ebb\left[|\scL(U,\psi, \lan \cdot,\cdot\ran_h)-\scL(\mu^N, \psi, \lan\cdot,\cdot\ran)|\right]\leq Ch.\]
\end{lemm}
\begin{proof}
Using the notation $f^\Cbf$ from Lemma \ref{lemm:histint} to denote piecewise constant approximation of a function $f$ over the domain $\CalD$ using the grid $\Cbf$, we have
\begin{align*}
\scL(U,\psi, \lan \cdot,\cdot\ran_h)-\scL(\mu^N, \psi, \lan\cdot,\cdot\ran) &= -\underbrace{\Big(\lan (\nabla \psi\cdot((\nabla K^\star)^\Cbf*\mu^N))^\Cbf,\mu^N\ran-\lan \nabla \psi\cdot\nabla K^\star*\mu^N,\mu^N\ran\Big)}_{E_{\text{interact}}}\\
&\qquad + \lan \partial_t \psi^\Cbf - \partial_t \psi, \mu^N\ran -\lan ((\nabla \psi\cdot\nabla V^\star)^\Cbf-\nabla \psi\cdot\nabla V^\star,\mu^N\ran \\
&\qquad + \frac{1}{2}\lan \text{Tr}\left(\nabla^2 \psi\sigma^\star (\sigma^\star)^T\right)^\Cbf-\text{Tr}\left(\nabla^2\psi\sigma^\star (\sigma^\star)^T\right), \mu^N\ran\\
&= E_{\text{interact}}+E_{\text{linear}}.
\end{align*}
The right-hand side includes an interaction error $E_{\text{interact}}$ followed by a sum $E_{\text{linear}}$ of terms that are linear in the difference between a locally Lipschitz function and its piecewise constant approximation. Hence, we can bound $E_{\text{linear}}$ using smoothness of $\psi$, the moment assumptions on $\mu^N_t$ and the growth assumptions on $V$ and $\sigma$. Specifically, for $x\in B_k$ with center $\cbf_k$, the growth assumptions imply
\begin{align*}
|\nabla \psi(x)\cdot \nabla V^\star(x)-\nabla \psi(\cbf_k)\cdot \nabla V^\star(\cbf_k)|&\leq Ch\Big((\nrm{\nabla\psi}_{2,\infty}+\text{Lip}(\nabla \psi))(1+|x|^p)\Big)\\
|\text{Tr}\left(\nabla^2 \psi(x)\sigma^\star(x) (\sigma^\star(x))^T\right)-\text{Tr}\left(\nabla^2 \psi(\cbf_k)\sigma^\star(\cbf_k) (\sigma^\star(\cbf_k))^T\right)|&\leq C'h\Big((\nrm{\nabla^2\psi}_{F,\infty}+\text{Lip}(\nabla^2\psi))(1+|x|^p)\Big)
\end{align*}
for $C,C'$ depending on $p,d$ and $C_p$, hence
\begin{equation}\label{elin}
|E_{\text{linear}}| \leq C''\sup_{|\alpha|\leq 3}\text{Lip}(\partial^\alpha\psi)\left(T+\int_0^T\int_{\Rbb^d}|x|^pd\mu^N_tdt \right)h.
\end{equation}
Similarly, for the interaction error we use that for $x\in B_k$ and $y\in B_j$ with centers $\cbf_k$ and $\cbf_j$, we have 
\begin{align*}
\left\vert\nabla \psi(\cbf_k)\cdot\nabla K^\star(\cbf_k-\cbf_j)-\nabla \psi(x)\cdot \nabla K^\star(x-y)\right\vert &\leq |\nabla \psi(\cbf_k)|\left\vert \nabla K^\star(\cbf_k-\cbf_j)-\nabla K^\star(x-y)\right\vert\\
&\qquad +\left\vert \nabla\psi(\cbf_k)-\nabla \psi(x)\right\vert\left\vert \nabla K^\star(x-y)\right\vert\\
&\leq C'''h\left(\nrm{\nabla \psi}_{2,\infty}+\text{Lip}(\nabla \psi)\right)(1+|x-y|^p)
\end{align*}
with $C'''$ also depending on $p$, $d$, and $C_p$. From this we have
\begin{equation}\label{einteract}
|E_\text{interact}|\leq C''''\left(T+\int_0^T\int_{\Rbb^d}\int_{\Rbb^d}|x-y|^pd\mu^N_t(y)d\mu^N_t(x)dt\right)h.
\end{equation}
The result follows from taking expectation and using the moment bound \eqref{momentbound}, where the final constant $C$ depends on $p,d,C_p,M_p,T,\eta$ and $\psi$.

\end{proof}



To incorporate discrete effects, we consider the difference between $\scL(U,\psi,\lan\cdot,\cdot\ran_h)$ and $\scL(U,\psi,\lan\cdot,\cdot\ran_{h,\Delta t})$, where recall that $\lan\cdot,\cdot\ran_{h,\Delta t}$ denotes trapezoidal rule integration in space with meshwidth $h$ and in time with sampling rate $\Delta t$.


\begin{lemm}\label{discreteeffects}
For $C$ independent of $N$, $h$ and $\Delta t$, it holds that
\[\Ebb\left[|\scL(U,\psi, \lan \cdot,\cdot\ran_h)-\scL(U, \psi, \lan\cdot,\cdot\ran_{h,\Delta t})|\right]\leq C(h+\Delta t^\eta).\]
\end{lemm}
\begin{proof}
Again rewriting the spatial trapezoidal-rule integration in the form $\int_{\Rbb^d}\varphi^\Cbf(x)d\mu^N_t$, we see that 
\begin{equation}\label{lemm34diff}
\scL(U,\psi, \lan \cdot,\cdot\ran_h)-\scL(U, \psi, \lan\cdot,\cdot\ran_{h,\Delta t})
\end{equation}
reduces to four terms of the form 
\[A(\varphi):=\frac{1}{N}\sum_{i=1}^N\left(\int_0^T \varphi^\Cbf(X^{(i)}_t)dt - \frac{\Delta t}{2}\sum_{\ell=1}^L \left(\varphi^\Cbf(X^{(i)}_{t_{\ell+1}})+\varphi^\Cbf(X^{(i)}_{t_\ell})\right)\right),\]
for $\varphi\in\left\{\partial_t\psi, \nabla \psi\cdot\nabla V^\star, \text{Tr}(\nabla^2\psi\sigma^\star(\sigma^\star)^T), \nabla\psi\cdot\nabla K^\star*\mu^N_t\right\}$. Similarly to the bounds derived for $|\varphi(x)-\varphi^\Cbf(x)|$ in Lemma \ref{diffcont}, the growth bounds on $V^\star,K^\star$ and $\sigma^\star$ imply in general that 
\[|\varphi(x)-\varphi(y)|\leq C|x-y|\left(1+\max\{|x|,|y|\}^p\right).\] 
Rewriting the summands in $A(\varphi)$,
\[\int_0^T \varphi^\Cbf(X^{(i)}_t)dt - \frac{\Delta t}{2}\sum_{\ell=1}^L \left(\varphi^\Cbf(X^{(i)}_{t_{\ell+1}})+\varphi^\Cbf(X^{(i)}_{t_\ell})\right)\]
\[=\sum_{\ell=1}^L\underbrace{\int_{t_\ell}^{t_{\ell+1}} \left(\frac{t-t_\ell}{\Delta t}\right)(\varphi^\Cbf(X^{(i)}_t)-\varphi^\Cbf(X^{(i)}_{t_{\ell+1}}))dt}_{I_1}+ \underbrace{\int_{t_\ell}^{t_{\ell+1}} \left(\frac{t_{\ell+1}-t}{\Delta t}\right)(\varphi^\Cbf(X^{(i)}_t)-\varphi^\Cbf(X^{(i)}_{t_{\ell}}))dt}_{I_2},\]
and using 
\[|\varphi^\Cbf(x)-\varphi^\Cbf(y)|\leq |\varphi(x)-\varphi(\cbf_k)|+|\varphi(x)-\varphi(y)|+|\varphi(y)-\varphi(\cbf_\ell)|\leq C(2h+|x-y|)(1+\max\{|x|,|y|\}^p)\]
where $x\in B_k$ and $y\in B_\ell$, we see that for $I_1$,
\[\left\vert\int_{t_\ell}^{t_{\ell+1}} \left(\frac{t-t_\ell}{\Delta t}\right)(\varphi^\Cbf(X^{(i)}_t)-\varphi^\Cbf(X^{(i)}_{t_{\ell+1}}))dt\right\vert\]
\[\leq \int_{t_\ell}^{t_{\ell+1}} \left(\frac{t-t_\ell}{\Delta t}\right)C(2h+|X_t^{(i)}-X^{(i)}_{t_{\ell+1}}|)(1+\max\{|X^{(i)}_t|,|X^{(i)}_{t_{\ell+1}}|\}^p)dt\]
\[\leq \int_{t_\ell}^{t_{\ell+1}} \left(\frac{t-t_\ell}{\Delta t}\right)C'(2h+|t_{\ell+1}-t|^\eta|)(1+\max\{|X^{(i)}_t|,|X^{(i)}_{t_{\ell+1}}|\}^p)dt.\]
Taking expectation on both sides and using the moment bound \eqref{momentbound}, we get
\[\Ebb\left[\left\vert\int_{t_\ell}^{t_{\ell+1}} \left(\frac{t-t_\ell}{\Delta t}\right)(\varphi^\Cbf(X^{(i)}_t)-\varphi^\Cbf(X^{(i)}_{t_{\ell+1}}))dt\right\vert\right]\leq C\left(\Delta t h+\Delta t^{1+\eta}\right).\]
We get the same bound for $I_2$. Summing over $\ell$, and taking the average in $i$, we then get 
\[\Ebb\left[|A(\varphi)|\right]\leq C(h+\Delta t^\eta),\]
which implies the desired bound on the difference \eqref{lemm34diff}.
\end{proof}
The previous estimates directly lead to the following bound on the model coefficients $\what$:


\begin{thm}\label{thm:w}
Let $\what=\Gbf^\dagger\bbf$ be the learned model coefficients and $\wstar$ the true model coefficients. For $C$ independent of $N,h$ and $\Delta t$ it holds that
\[\Ebb\left[\nrm{\what-\wstar}_1\right] \leq C\left(h+N^{-1/2}+\Delta t^\eta\right).\]
\end{thm}
\begin{proof}
Using that $K^\star$, $V^\star$ and $\sigma^\star$ are in the span of $\Lbb$, we have that 
\[\bbf_k = \lan \partial_t\psi_k,\Ubf\ran_{h,\Delta t} = \scL(U,\psi_k, \lan \cdot,\cdot \ran_{h,\Delta t}) + \Gbf^T_k \wstar:=\Lbf_k +\Gbf^T_k\wstar,\]
where $\Gbf^T_k$ is the $k$th row of $\Gbf$. From the previous lemmas, we have
\[\Ebb\left[|\Lbf_k|\right]\leq \Ebb\left[|\scL(U,\psi_k, \lan \cdot,\cdot\ran_{h,\Delta t})-\scL(U, \psi_k, \lan\cdot,\cdot\ran_h)|\right]\]
\[+\Ebb\left[|\scL(U,\psi_k, \lan \cdot,\cdot\ran_h)-\scL(\mu^N, \psi_k, \lan\cdot,\cdot\ran)|\right]+\Ebb\left[|\scL(\mu^N, \psi_k, \lan\cdot,\cdot\ran)|\right]\]
\[\leq C'\left(h + N^{-1/2}+\Delta t^\eta\right).\]
Using that $\Gbf$ is full rank, it holds that $\what=\Gbf^\dagger\bbf=\Gbf^\dagger\Lbf +\wstar$, hence the result follows from the uniform bound on $\nrm{\Gbf^\dagger}_1$:
\[\Ebb\left[\nrm{\what-\wstar}_1\right]\leq \nrm{\Gbf^\dagger}_2\Ebb\left[\nrm{\Lbf}_1\right]\leq C'n\nrm{\Gbf^\dagger}_1\left(h + N^{-1/2}+\Delta t^\eta\right).\]
\end{proof}
Under the assumption that $K^\star, V^\star$ and $\sigma^\star$ are contained in the span of $\Lbb$, an immediate corollary is 
\[\Ebb\left[\nrm{K^\star- \widehat{K}}_{L^2(\CalD-\CalD)}+\nrm{V^\star - \widehat{V}}_{L^2(\CalD)}+\nrm{\nrm{\sigma^\star(\sigma^\star)^T-\widehat{\sigma}(\widehat{\sigma})^T}_F}_{L^2(\CalD)}\right]\leq C\left(h+N^{-1/2}+\Delta t^\eta\right).\]
Finally, setting $h = N^{-\alpha}$ for $\alpha>0$ will ensure convergence as $N\to\infty$ and $\Delta t\to 0$. We now make several remarks about the practical (algorithmic) implementation with respect to this theoretical convergence.

\begin{rmrk}
\textcolor{white}{blah}\\ \vspace{-0.5cm}
\begin{itemize}
\item An important case of Theorem \ref{thm:w} is $\sigma^\star=0$, in which case $\mu^N_t$ itself is a weak-measure solution to the mean-field equation \eqref{fpmeanfield} and $\nrm{\what-\wstar}_2 \leq C(h+\Delta t^\eta)$, with $\eta\geq 2$. Although the examples below only explore $\sigma^\star=0$ with nonzero extrinsic noise $\varepsilon$ (Figures \ref{qanrsig-Inf} and \ref{log2D_nu-Inf}), we note that when $\varepsilon=0$ and $\sigma^\star=0$ (not shown) Algorithm \ref{alg1} recovers systems to high accuracy similarly to WSINDy applied to local dynamical systems \cite{messenger2020weak,messenger2020weakpde}.  
\item Algorithm \ref{alg:wsindypde} in general implements sparse regression, yet Theorem \ref{thm:w} deals with ordinary least squares. In practice, sparse regression provides an improvement over least squares since $\Gbf$ typically has a high condition number (e.g.\ due to coupled effects of $K$,$V$, and $\sigma$). Since least squares is a common subroutine of many sparse regression algorithms (inluding the MSTLS algorithm used here), the result is still relevant to sparse regression. Lastly, the full-rank assumption on $\Gbf$ implies that as $N\to \infty$ sequential thresholding reduces to least squares. 
\item Theorem \ref{thm:w} assumes data from a single experiment ($M=1$), while the examples below show that $M>1$ experiments improves results. For any fixed $M>1$, the $N\to\infty$ limit results in convergence, however, the $N$-fixed and $M\to \infty$ limit does not result in convergence, as this does not lead to the mean-field equations\footnote{Note that the opposite convergence holds for the algorithm introduced in \cite{lu2020learning}: $N$-fixed, $M\to \infty$ results in recovery of $K$.}. The examples below show that using $M>1$ has a practical advantage. 
\item Many interesting examples have non-Lipschitz $\nabla K$, in particular a lack of smoothness at $x=0$. If $\nabla K$ has a jump discontinuity at $x=0$ and $\mu^N_t$ does not concentrate to a singular measure as $N\to \infty$, then the bound \eqref{einteract} may be modified to include another $\CalO(h)$ term coming from short-range interactions (i.e.\ within an $\CalO(h)$ distance). The examples below are chosen in part to show that $\CalO(N^{-1/2})$ convergence holds for $\nabla K$ with jumps at the origin.  
\end{itemize}
\end{rmrk}

\section{Examples}\label{sec:examples}

We now demonstrate the successful identification of several particle systems in one and two spatial dimensions as well as the $\CalO(N^{-1/2})$ convergence predicted in Theorem \ref{thm:w}. In each case we use Algorithm \ref{alg1} to discover a mean-field equation of the form \eqref{fpmeanfield} from discrete-time particle data. For each dataset we simulate the associated interacting particle system $\Xbf_t$ given by \eqref{dXt} using the Euler-Maruyama scheme (initial conditions and timestep are given in each example). We assess the ability of WSINDy to select the correct model using the \textit{true positivity ratio}\footnote{For example, identification of the true model ($\supp{\what} =\supp{\wstar}$) results in a TPR($\what)=1$, while identification of only half of the correct nonzero terms and no additional falsely identified terms results in TPR($\what)=0.5$.} 
\begin{equation}\label{tpr}
\text{TPR}(\what) = \frac{\text{TP}}{\text{TP}+\text{FN}+\text{FP}}
\end{equation}
where TP is the number of correctly identified nonzero coefficients, FN is the number of coefficients falsely identified as zero, and FP is the number of coefficients falsely identified as nonzero \cite{LagergrenNardiniMichaelLavigneEtAl2020ProcRSocA}. To demonstrate the $\CalO(N^{-1/2})$ convergence, for correctly identified models (i.e.\ TPR$(\what)=1$) we compute the relative $\ell_2$ error in the recovered interaction force $\nabla \widehat{K}$, local force $\nabla \widehat{V}$, and diffusivity $\widehat{\sigma}$ over $\Cbf-\Cbf$ and $\Cbf$, respectively. Results are averaged over 100 trials.\\

For the computational grid $\Cbf$ we first compute the sample standard deviation $s$ of $\pmb{\Ybb}$ and we choose $\CalD$ to be the rectangular grid extending $3s$ from the mean of $\pmb{\Ybb}$ in each direction. We then set $\Cbf$ to have 128 points in $x$ and $y$ for $d=2$ dimensions, and 256 points in $x$ for $d=1$, noting that these numbers are fairly arbitrary, and used to show that the grid need not be too large. We set the sparsity factors so that $\log_{10}(\pmb{\lambda})$ contains 100 equally spaced points from $-4$ to 0. More information on the specifications of each example can be found in Appendix \ref{app:specs}. 

\subsection{Two-Dimensional Local Model}\label{sec:cos2D}

The first system we examine is a constant advection / variable diffusivity model with mean-field equation\footnote{Since the model is local, \eqref{cos2Deq} is the Fokker-Planck equation for the distribution of each particle, rather than only in the limit of infinite particles.}
\begin{equation}\label{cos2Deq}
\partial_t\mu_t =-\partial_x\mu_t-\partial_y\mu_t + \Delta \left[\left(1+0.95\cos(\omega x)\cos(\omega y)\right)\mu_t\right].
\end{equation}
The purpose of this example is three-fold. First, we are interested in the ability of Algorithm \ref{alg1} to correctly identify a local model from a library containing both local and nonlocal terms. Next, we evaluate whether the $\CalO(N^{-1/2})$ convergence is realized. Lastly, we investigate whether for large $\omega$ the weak-form identifies the associated homogenized equation  (see e.g.\ \cite{weinan2011principles}) 
\begin{equation}\label{cos2Deq2}
\partial_t\mu_t =-\partial_x\mu_t-\partial_y\mu_t + \overline{\omega}\Delta \mu_t,
\end{equation}
where $\overline{\omega}$ is given by the harmonic mean of diffusivity:
\[\overline{\omega} = \left(\int_\CalD \frac{dxdy}{1+0.95\cos(x)\cos(y)}\right)^{-1}.\]

For $\omega\in\{1,20\}$ we evolve the particles from an initial Gaussian distribution with mean zero and covariance $\Ibf_2$ and record particle positions for $100$ timesteps with $\Delta t=0.02$ (subsampled from a simulation with timestep $10^{-4}$). We use a rectangular domain $\CalD$ of approximate sidelength $10$ and compute histograms with 128 bins in $x$ and $y$ for a spatial resolution of $\Delta x\approx 0.078$ (see Figure \ref{cos2Dsol} for solution snapshots), over which $\overline{\omega}\approx 0.62$. For $\omega=1$ we compare recovered equations with the full model \eqref{cos2Deq}, while for $\omega=20$ we compare with \eqref{cos2Deq2}, for comparison computing $\overline{\omega}$ over each domain $\CalD$ using MATLAB's \texttt{integral2}. Figure \ref{cos2Dresults} shows that as the particle number increases, we do in fact recover the desired equations, with TPR$(\what)$ approaching one as $N$ increases. For $\omega=1$ we observe $\CalO(N^{-1/2})$ convergence of the local potential $\widehat{V}$ and the diffusivity $\widehat{\sigma}$. For $\omega=20$, we observe approximate $\CalO(N^{-1/2})$ convergence of $\widehat{V}$, and $\widehat{\sigma}$ converging to within $2\%$ of $\sqrt{2\overline{\omega}}$, the homogenized diffusivity (higher accuracy can hardly be expected for $\omega=20$ since \eqref{cos2Deq2} is itself an approximation in the limit of infinite $\omega$). 

\begin{figure}
\centering
\begin{tabular}{ccc}
\hspace{-0.25cm} \includegraphics[trim={10 5 25 20},clip,width=0.44\textwidth]{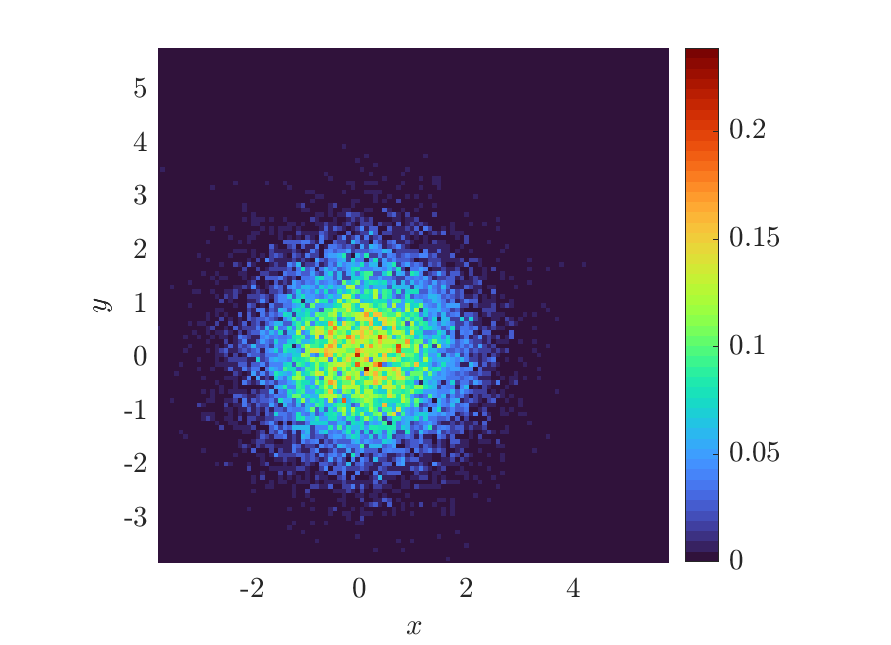} &
\hspace{-0.25cm} \includegraphics[trim={10 5 25 20},clip,width=0.44\textwidth]{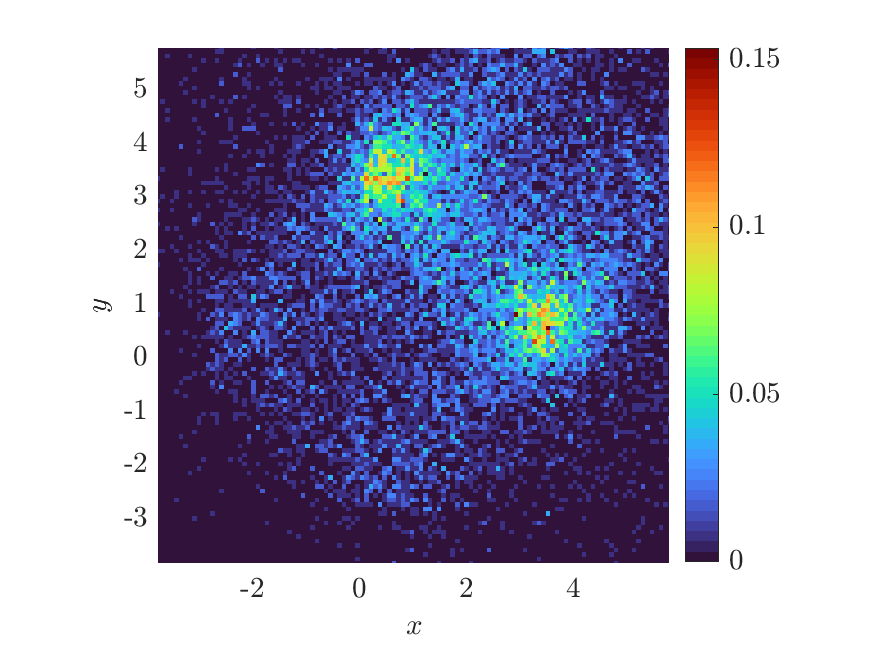} \\
\hspace{-0.25cm}	\includegraphics[trim={10 5 25 20},clip,width=0.44\textwidth]{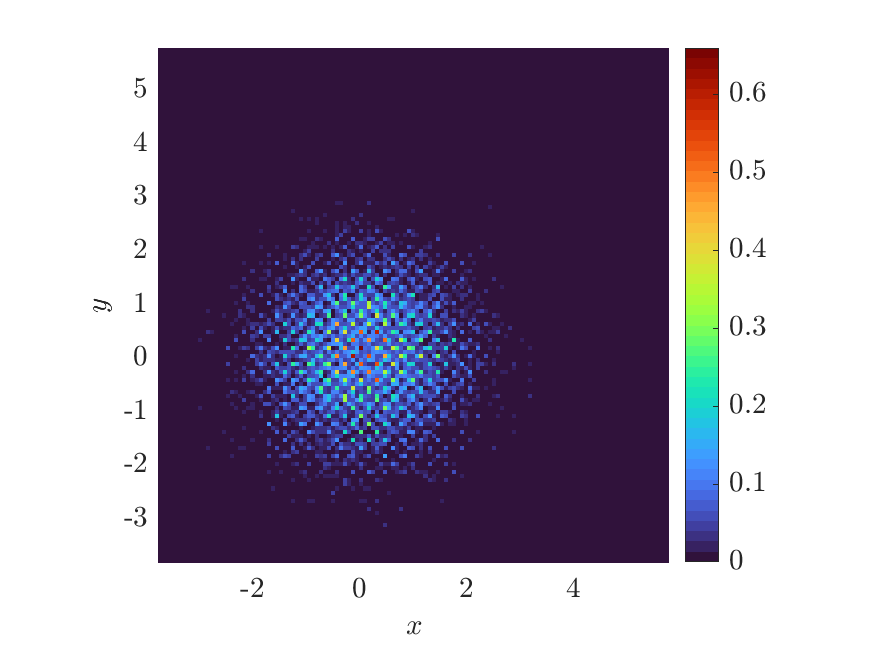}&
\hspace{-0.25cm}	\includegraphics[trim={10 5 25 20},clip,width=0.44\textwidth]{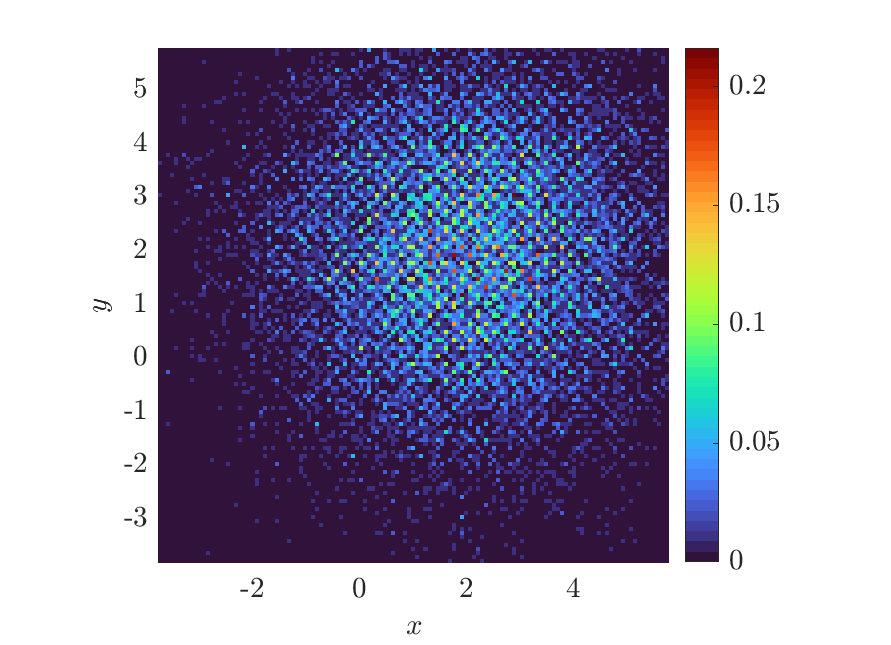}
\end{tabular}
\caption{Snapshots at time $t=2\Delta t=0.06$ (left) and $t=100\Delta t=2$ (right) of histograms computed with 128 bins in $x$ and $y$ from 16,384 particles evolving under \eqref{cos2Deq} with $\omega=1$ (top) and $\omega=20$ (bottom).}
\label{cos2Dsol}
\end{figure}

\begin{figure}
\centering
\begin{tabular}{ccc}
\hspace{-0.5cm}	\includegraphics[trim={0 5 25 10},clip,width=0.3\textwidth]{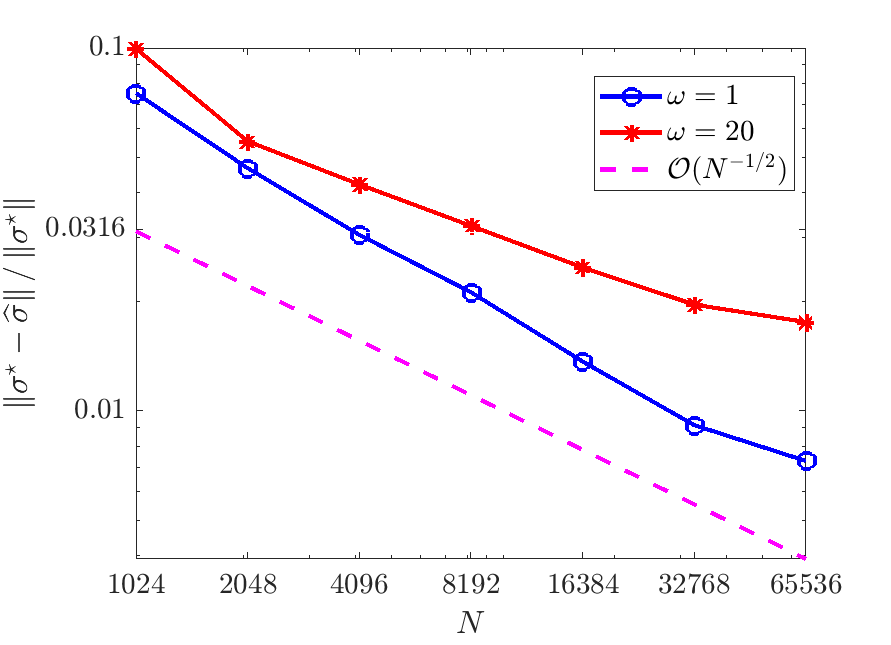}&
\hspace{-0.5cm}	\includegraphics[trim={0 5 25 10},clip,width=0.3\textwidth]{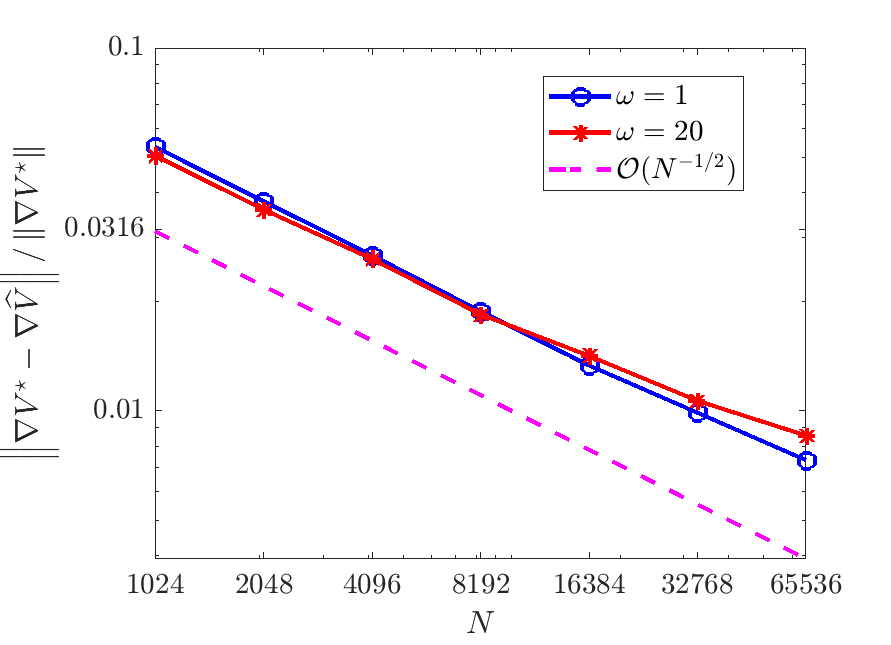}&
\hspace{-0.25cm}\includegraphics[trim={0 5 25 10},clip,width=0.3\textwidth]{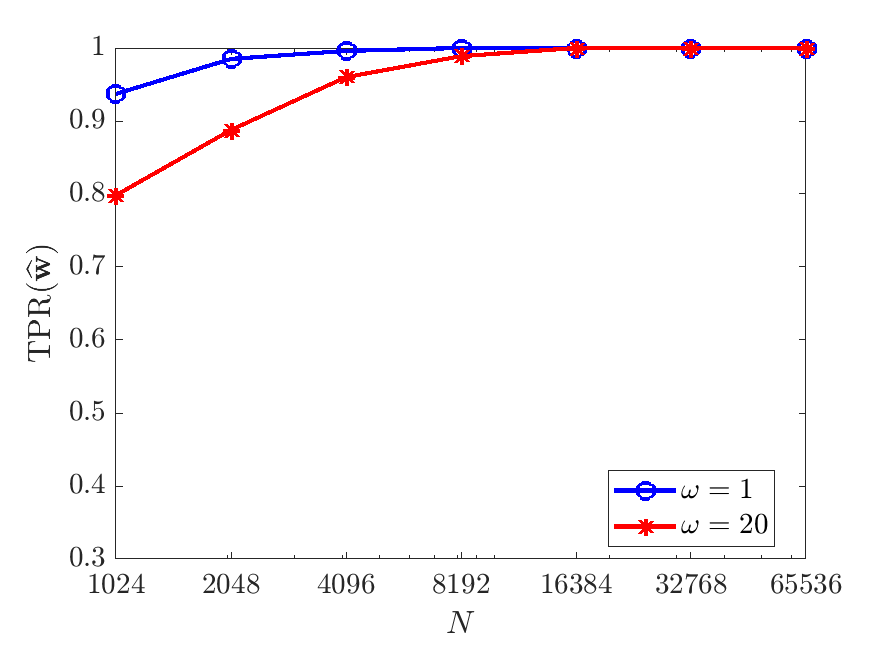} 
\end{tabular}
\caption{Convergence of $\widehat{\sigma}$ (left) and $\nabla \widehat{V}$ (middle) for \eqref{cos2Deq} with $\omega\in\{1,20\}$, as well as TPR$(\what)$ (right). For $\omega=1$, comparison is made with respect to the exact model \eqref{cos2Deq}, while for $\omega=20$ results are compared to \eqref{cos2Deq2}.}
\label{cos2Dresults}
\end{figure}

\subsection{One-Dimensional Nonlocal Model}\label{sec:qanr}

\begin{figure}
\centering
\begin{tabular}{ccc}
\hspace{-0.25cm}	\includegraphics[trim={10 5 25 20},clip,width=0.3\textwidth]{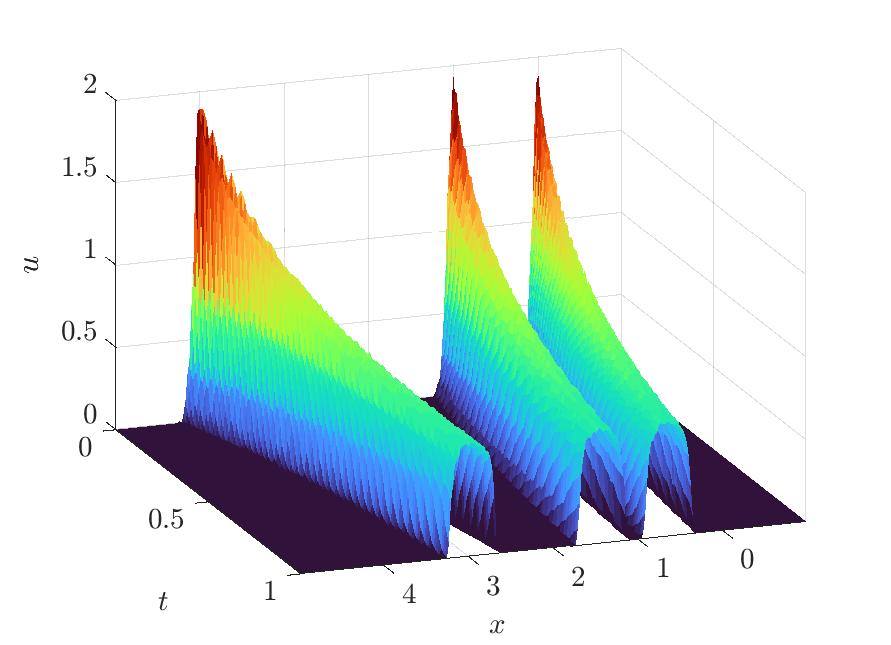} &
\hspace{-0.5cm}	\includegraphics[trim={10 5 25 20},clip,width=0.3\textwidth]{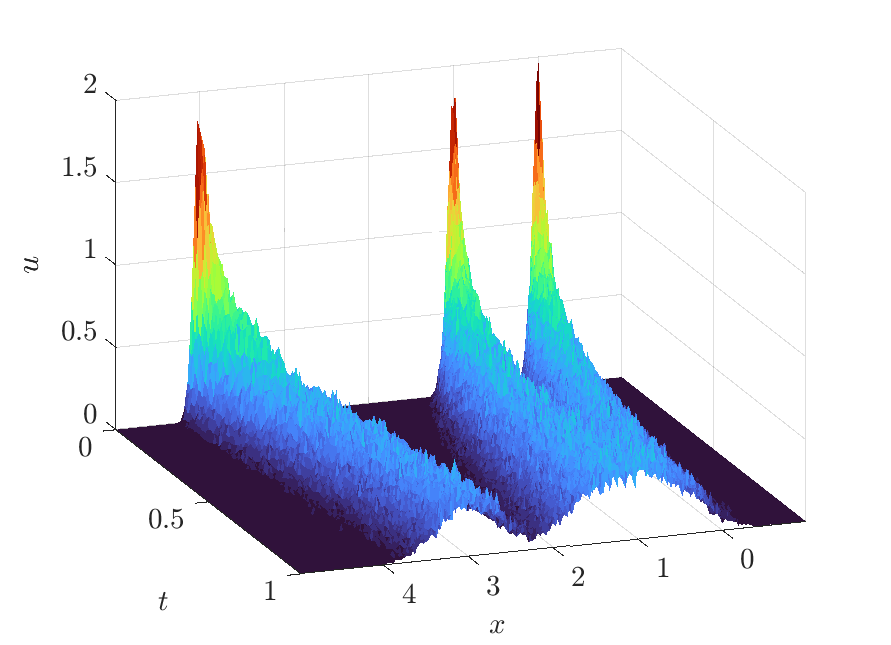}
&
\hspace{-0.5cm}	\includegraphics[trim={10 5 25 20},clip,width=0.3\textwidth]{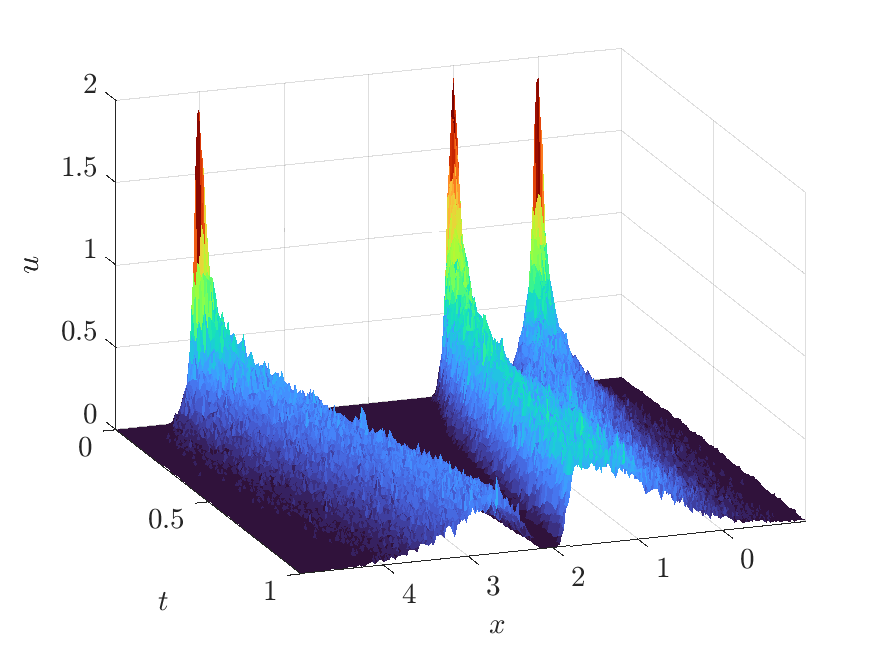}\\
\hspace{-0.5cm}	\includegraphics[trim={10 5 25 20},clip,width=0.3\textwidth]{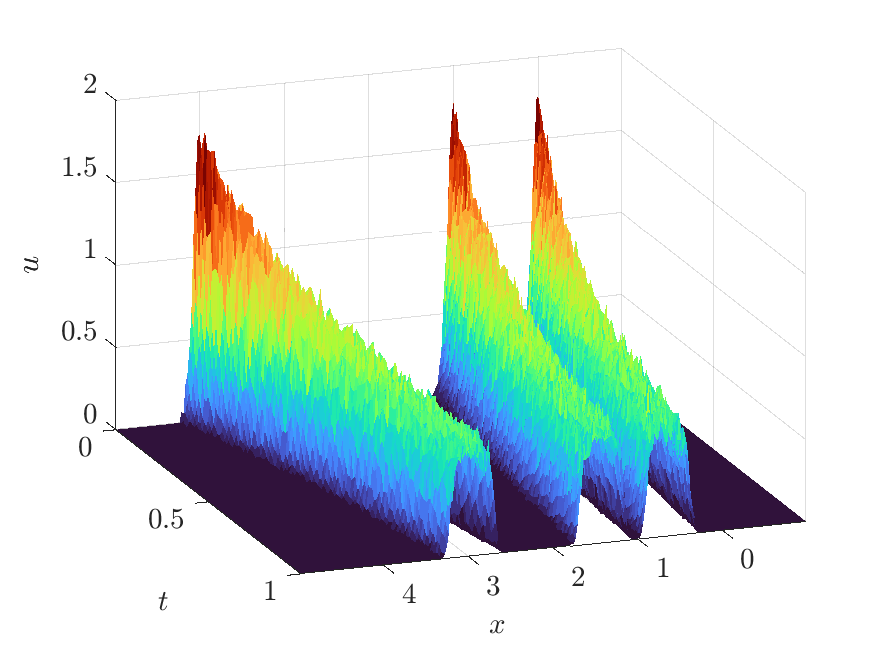} &
\hspace{-0.5cm}	\includegraphics[trim={10 5 25 20},clip,width=0.3\textwidth]{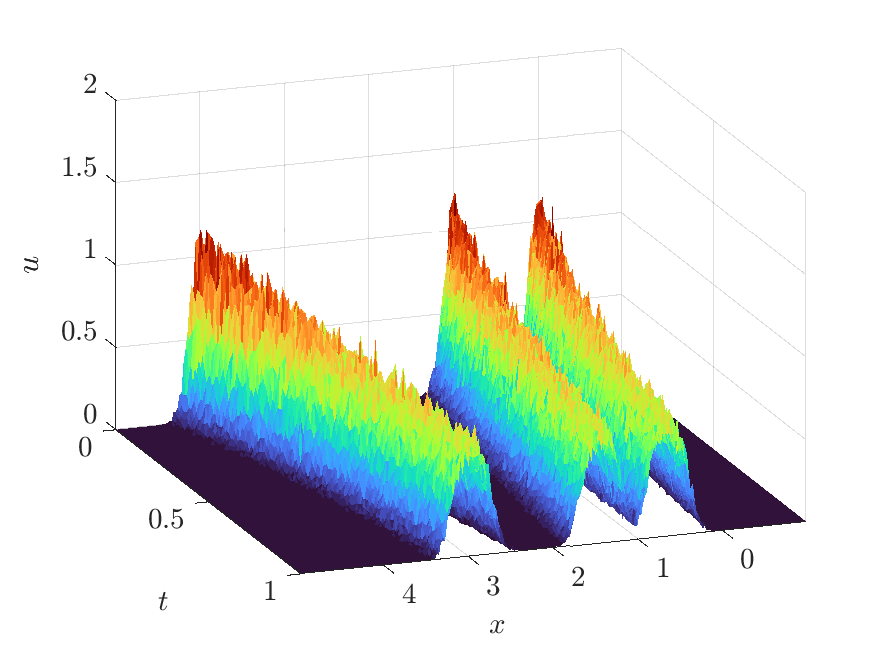}
&
\hspace{-0.5cm}	\includegraphics[trim={10 5 25 20},clip,width=0.3\textwidth]{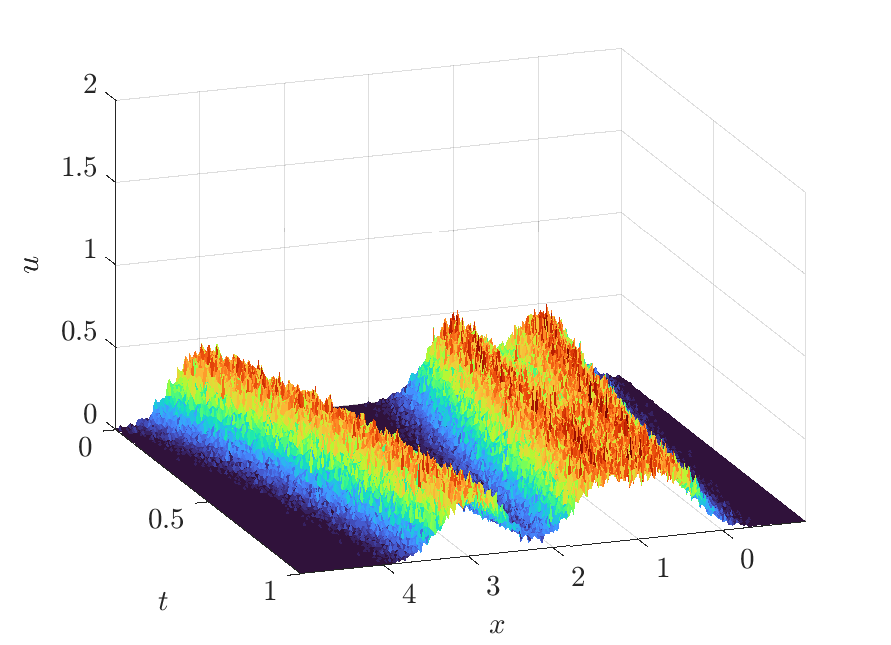}
\end{tabular}
\caption{Histograms computed with 256 bins width $h=0.0234$ from 8000 particles in 1D evolving under $K^\star=K_{\text{QANR}}(x)$  \eqref{qanr}. Top left to top right: $\sigma^\star(x)=0$, $\sigma^\star(x)=\sqrt{2(0.1)}$, $\sigma^\star(x) = \sqrt{2(0.1)}|x-2|$. Bottom: deterministic particles with i.i.d.\ Gaussian noise added to particle positions with resulting noise ratios (left to right) $\ep=0.0316, 0.1, 0.316$.}
\label{qanrsol}
\end{figure}

We simulate the evolution of particle systems under the quadratic attraction / Newtonian repulsion potential
\begin{equation}\label{qanr}
K_{\text{QANR}}(x) = \frac{1}{2} x^2 -|x|
\end{equation}
with no external potential $(V=0)$. The $-|x|$ portion of $K_{\text{QANR}}$ , leading to a discontinuity in $\nabla K$, is the one-dimensional free-space Green's function for $-\Delta$. For $d\geq 1$, when replaced by the corresponding Green's function in $d$ dimensions, the distribution of particles evolves under $K_{\text{QANR}}$ into the characteristic of the unit ball in $\Rbb^d$, which has implications for design and control of autonomous systems \cite{fetecau2011swarm}. We compare three diffusivity profiles, $\sigma(x)=0$ corresponding to zero intrinsic noise, $\sigma(x)=\sqrt{2(0.1)}$ leading to constant-diffusivity intrinsic noise, and $\sigma(x) = \sqrt{2(0.1)}|x-2|$ leading to variable-diffusivity intrinsic noise. With zero intrinsic noise ($\sigma(x)=0$), we examine the effect of extrinsic noise on recovery, and assume uncertainty in the particle positions due to measurement noise at each timestep, $\pmb{\Ybb} = \pmb{\Xbb} + \varepsilon$, 
for $\varepsilon \sim \CalN(0,\ep^2\nrm{\Xbf_\tbf}_{\text{RMS}}^2)$ i.i.d. and $\ep\in \{0.01,0.0316,0.1,0.316\}$. In this way $\ep$ is the {\it noise ratio}, such that $\nrm{\varepsilon}_F/\nrm{\pmb{\Xbb}}_F\approx\ep$ (computed with $\varepsilon$ and $\pmb{\Xbb}$ stretched into column vectors).\\

Measurement data consists of 100 timesteps at resolution $\Delta t =0.01$, coarsened from simulations with timestep $0.001$. Initial particle positions are drawn from a mixture of three Gaussians each with standard deviation $0.005$. Histograms are constructed with 256 bins of width $h=0.0234$. Typical histograms for each noise level are shown in Figure \ref{qanrsol} computed one experiment with $N=8000$ particles.\\

For the case of extrinsic noise (Figure \ref{qanrsig-Inf}), we use only one experiment ($M=1$) and examine the number of particles $N$ and the noise ratio $\ep$. We find that recovery is accurate and reliable for $\ep\leq 0.1$, yielding correct identification of $K_{\text{QANR}}$ with less than $1\%$ relative error in at least $98/100$ trials. Increasing $N$ from 500 to 8000 leads to minor improvements in accuracy for $\ep\leq 0.1$, but otherwise has little effect, implying that for low to moderate noise levels the mean field equations are readily identifiable even from smaller particle systems. For $\ep=10^{-1/2}\approx 0.3162$ (see Figure \ref{qanrsol} (bottom right) for an example histogram), we observe a decrease in TPR($\what$) (Figure \ref{qanrsig-Inf} middle panel) resulting from the generic identification of a linear diffusion term $\nu\partial_{xx}u$ with $\nu\approx 0.05$. Using that $\sqrt{2\nu}\approx\sqrt{2(0.05)}=\ep$, we can identify this as the best-fit {\it intrinsic} noise model. Furthermore, increases in $N$ lead to reliable identification of the drift term, as measured by TPR($\what_{drift}$) (rightmost panel Figure \ref{qanrsig-Inf}) which is the restriction of TPR to drift terms $\Lbb_K$ and $\Lbb_V$.\\

For constant diffusivity $\sigma(x)=\sqrt{2(0.1)}$ (Figure \ref{qanrsig-1}), the full model is recovered with less than $3\%$ errors in $\hat{K}$ and $\hat{\sigma}$ in at least 98/100 trials when the total particle count $NM$ is at least $8000$, and yields errors less than $1\%$ for $NM\geq$ 16,000. The error trends for $\hat{K}$ and $\hat{\sigma}$ in this case both strongly agree with the predicted $\CalO(N^{-1/2})$ rate. For non-constant diffusivity $\sigma(x)=\sqrt{2(0.1)}|x-2|$ (Figure \ref{qanrsig-sigx}), we also observe robust recovery (TPR$(\what)\geq 0.95$) for $NM\geq 8000$ with error trends close to $\CalO(N^{-1/2})$, although the accuracy in $\hat{K}$ and $\hat{\sigma}$ is diminished due to the strong order $\Delta t^{1/2}$ convergence of Euler-Maruyama applied to diffusivities $\sigma$ that are unbounded in $x$ \cite{milstein1994numerical}.

\begin{figure}
\centering
\begin{tabular}{ccc}
	\includegraphics[trim={0 0 25 15},clip,width=0.33\textwidth]{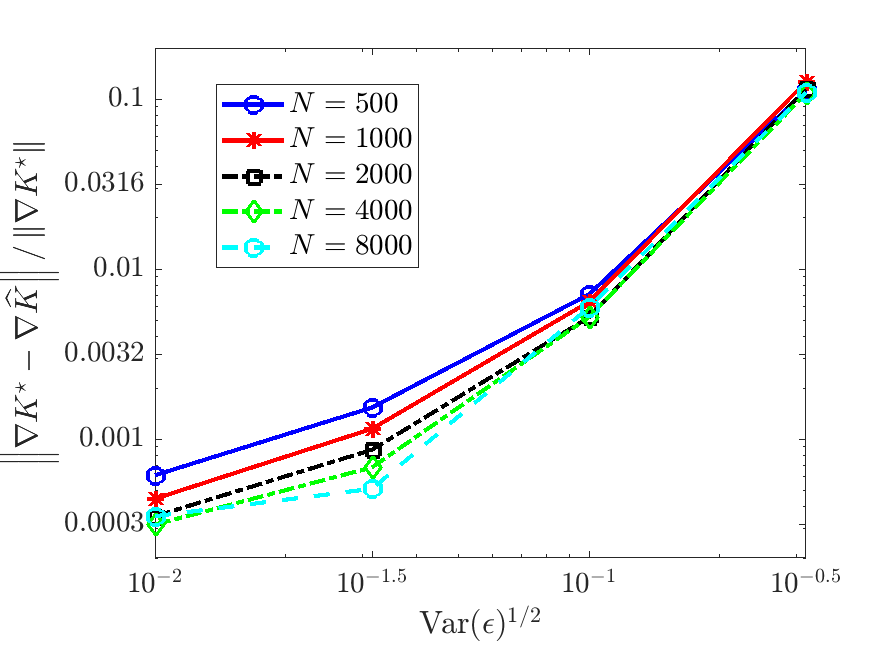} &
\hspace{-0.5cm}	\includegraphics[trim={0 0 25 15},clip,width=0.33\textwidth]{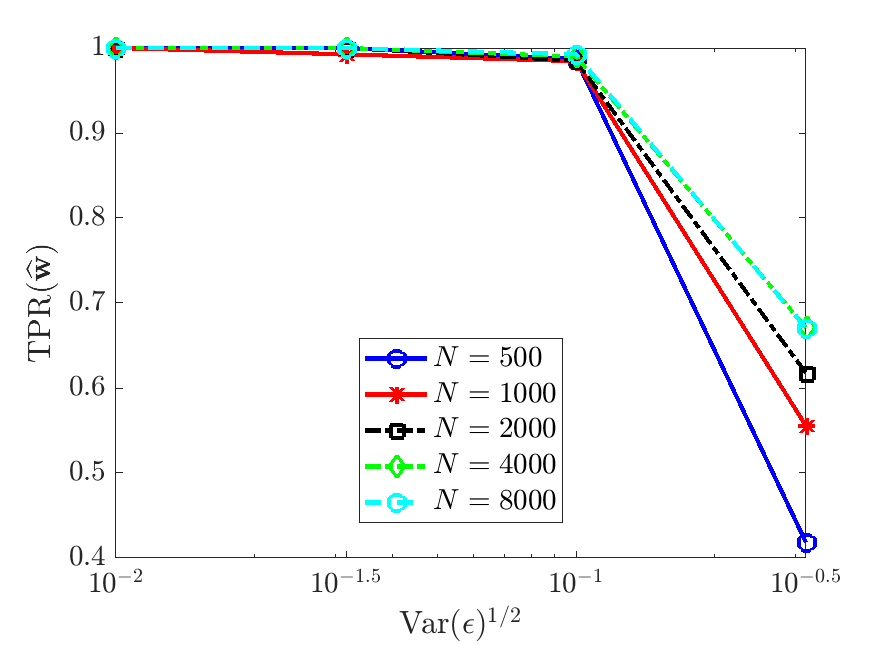} &
\hspace{-0.5cm}	\includegraphics[trim={0 0 25 15},clip,width=0.33\textwidth]{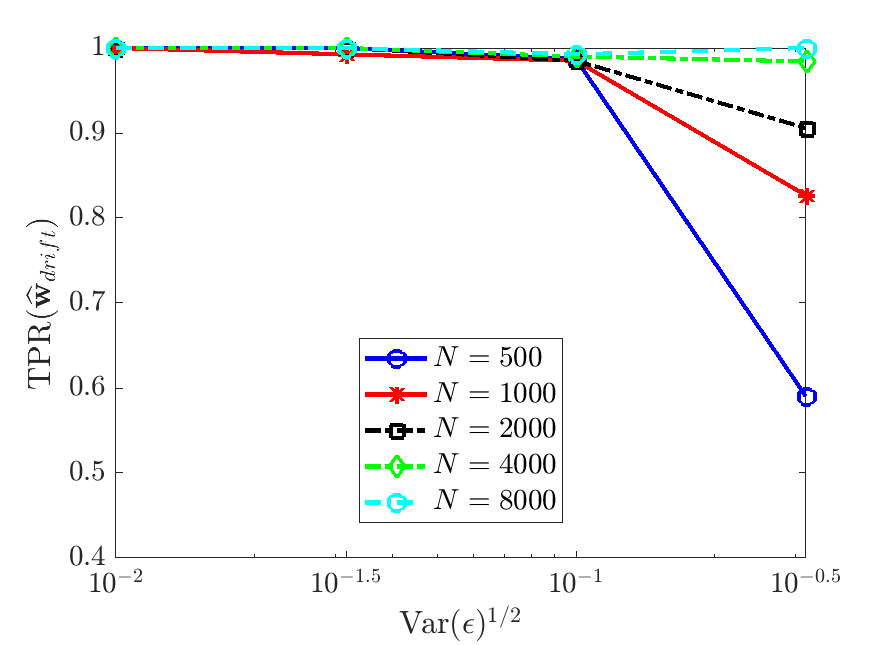}
\end{tabular}
\caption{Recovery of \eqref{fpmeanfield} in one spatial dimension for $K^\star=K_{\text{QANR}}$ and $\sigma^\star=0$ under different levels of observational noise $\ep$. Left: relative error in learned interaction kernel $\widehat{K}$. Middle: true positivity ratio for full model \eqref{fpmeanfield}. Right: true positivity ratio for drift term.}
\label{qanrsig-Inf}
\end{figure}

\begin{figure}
\centering
\begin{tabular}{ccc}
	\includegraphics[trim={0 0 25 15},clip,width=0.33\textwidth]{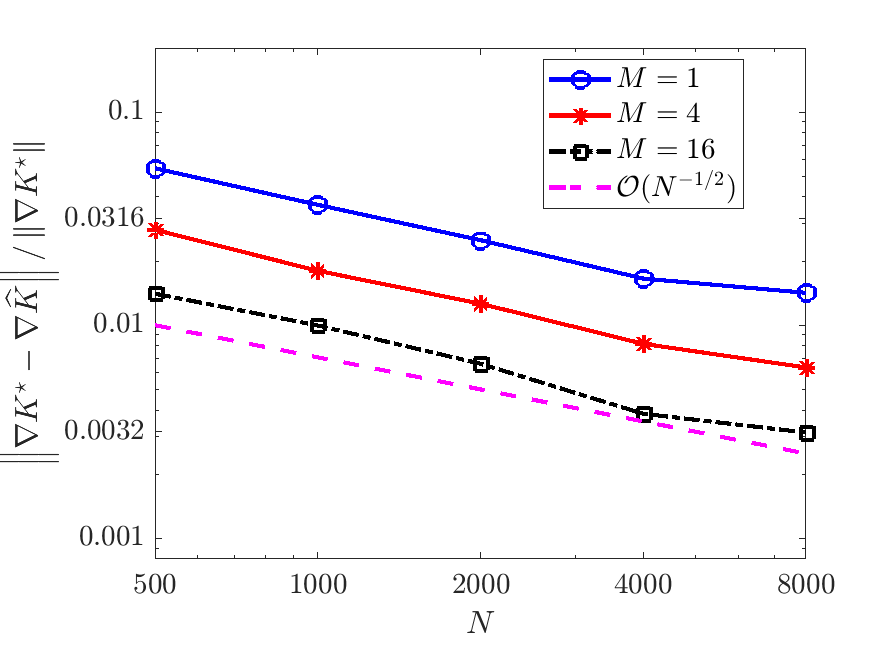} &
\hspace{-0.5cm}	\includegraphics[trim={0 0 25 15},clip,width=0.33\textwidth]{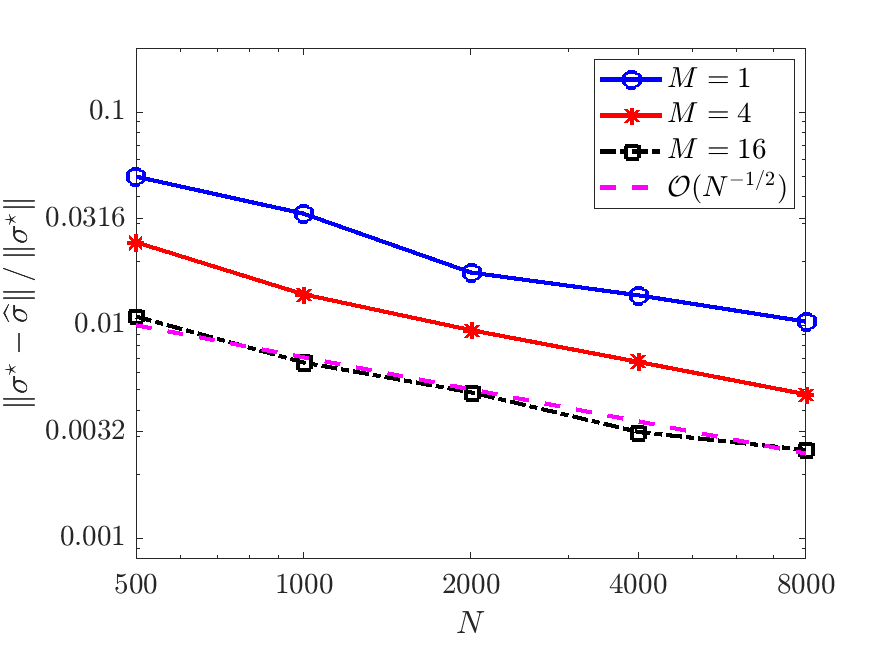} &
\hspace{-0.5cm}	\includegraphics[trim={0 0 25 15},clip,width=0.33\textwidth]{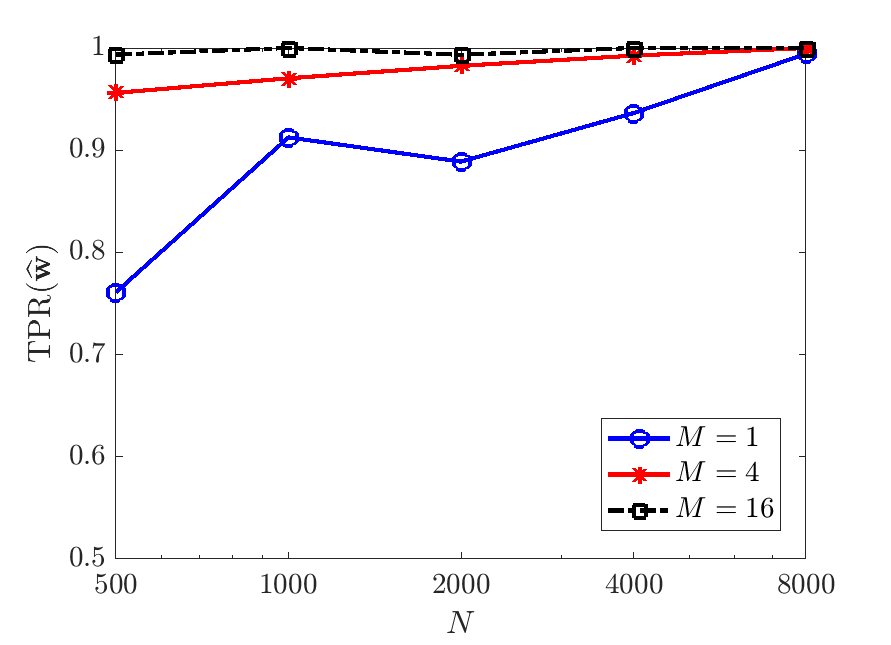}
\end{tabular}
\caption{Recovery of \eqref{fpmeanfield} in one spatial dimension for $K^\star=K_{\text{QANR}}$ and $\sigma^\star=\sqrt{2(0.1)}$}
\label{qanrsig-1}
\end{figure}

\begin{figure}
\centering
\begin{tabular}{ccc}
	\includegraphics[trim={0 0 25 15},clip,width=0.33\textwidth]{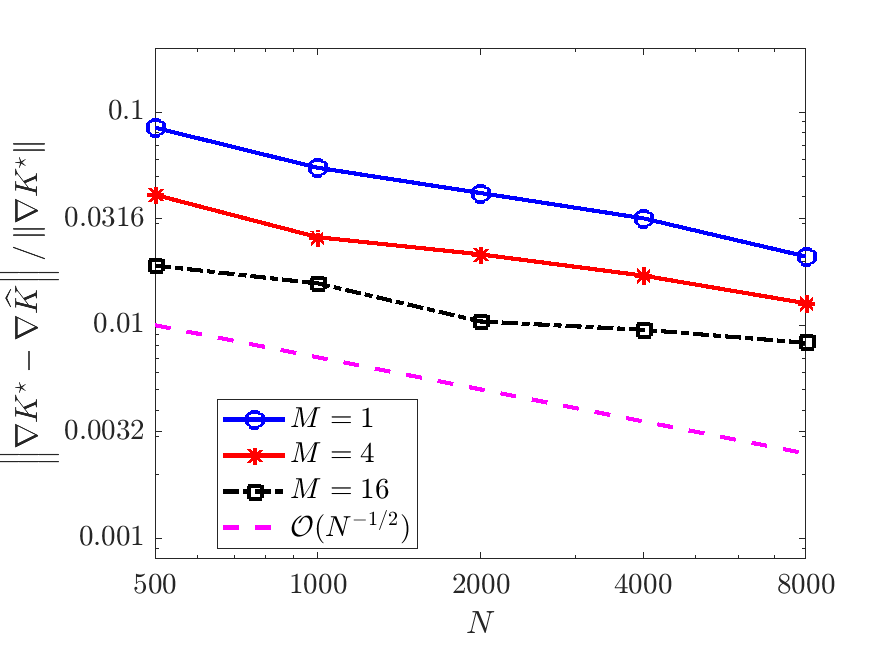} &
\hspace{-0.5cm}	\includegraphics[trim={0 0 25 15},clip,width=0.33\textwidth]{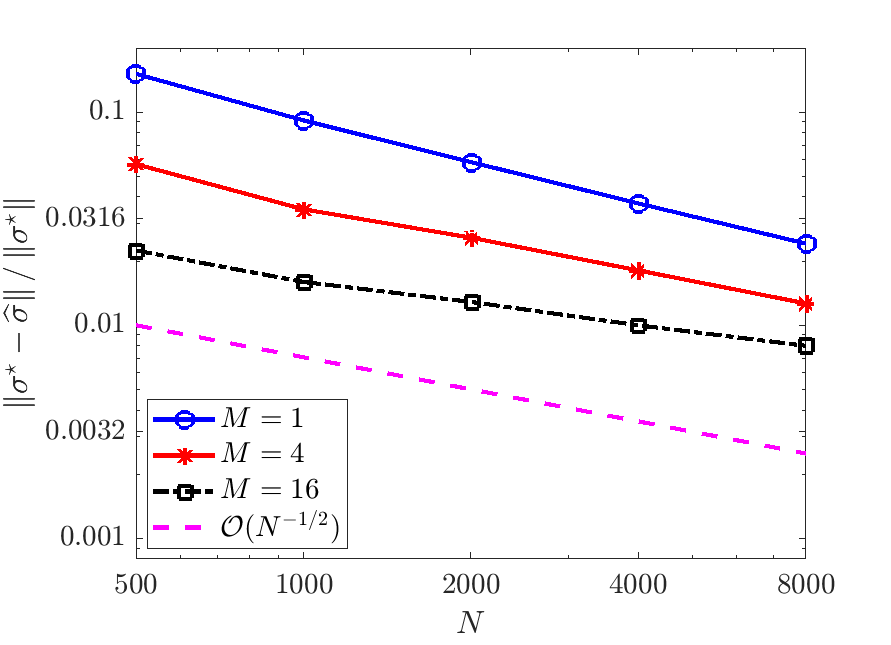} &
\hspace{-0.5cm}	\includegraphics[trim={0 0 25 15},clip,width=0.33\textwidth]{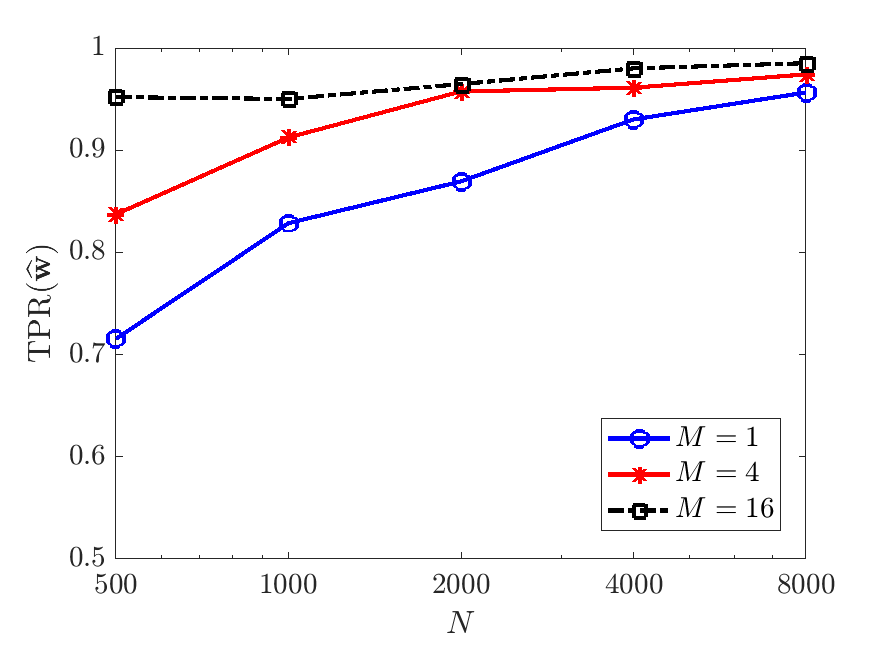}
\end{tabular}
\caption{Recovery of \eqref{fpmeanfield} in one spatial dimension for $K^\star=K_{\text{QANR}}$ and $\sigma^\star=\sqrt{2(0.1)|x-2|}$}
\label{qanrsig-sigx}
\end{figure}

\subsection{Two-Dimensional Nonlocal Model}\label{sec:log2D}

\begin{figure}
\centering
\begin{tabular}{ccc}
	\includegraphics[trim={38 5 40 20},clip,width=0.33\textwidth]{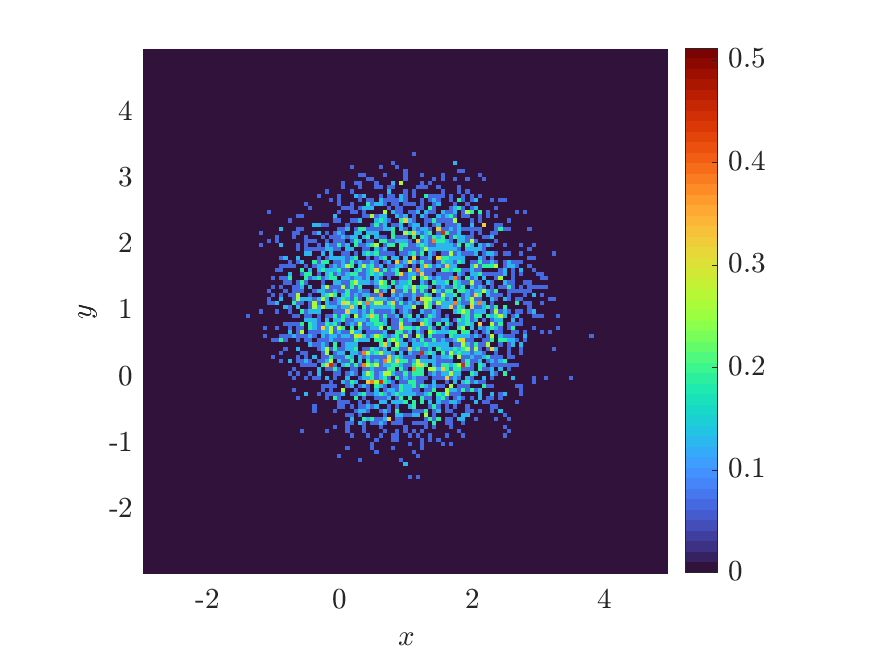} &
\hspace{-0.5cm}	\includegraphics[trim={38 5 40 20},clip,width=0.33\textwidth]{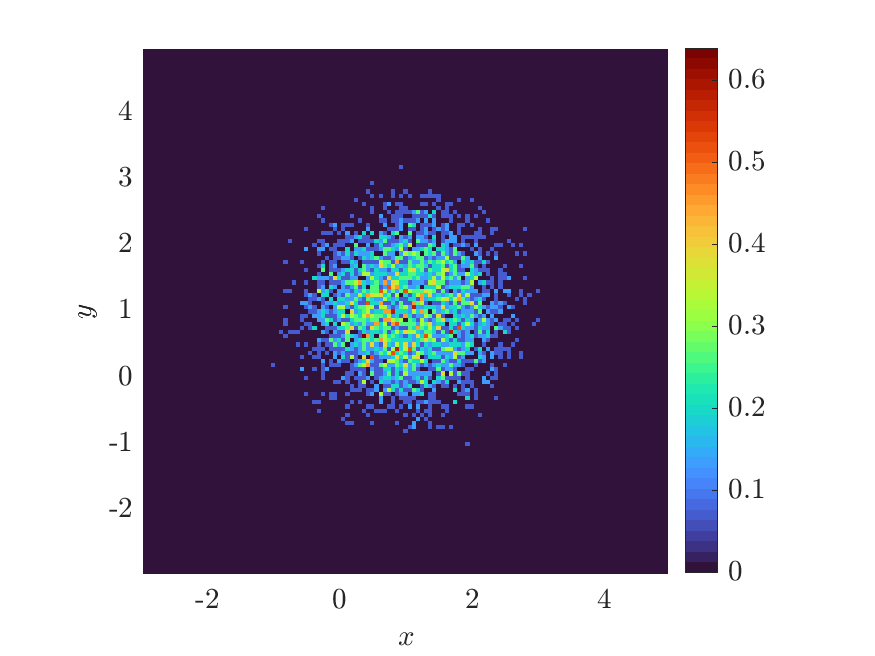} &
\hspace{-0.5cm}	\includegraphics[trim={38 5 40 20},clip,width=0.33\textwidth]{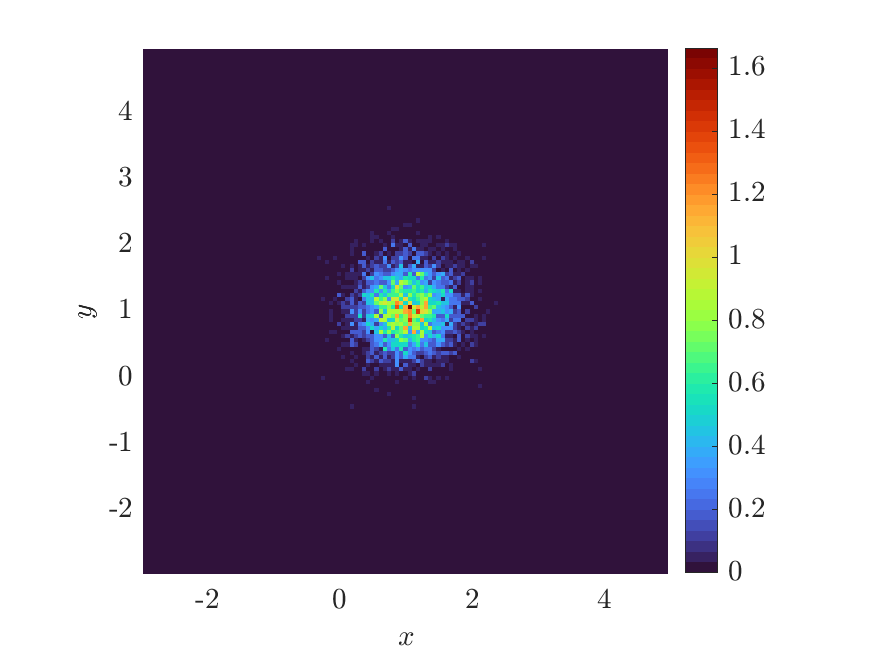} \\
	\includegraphics[trim={38 5 40 20},clip,width=0.33\textwidth]{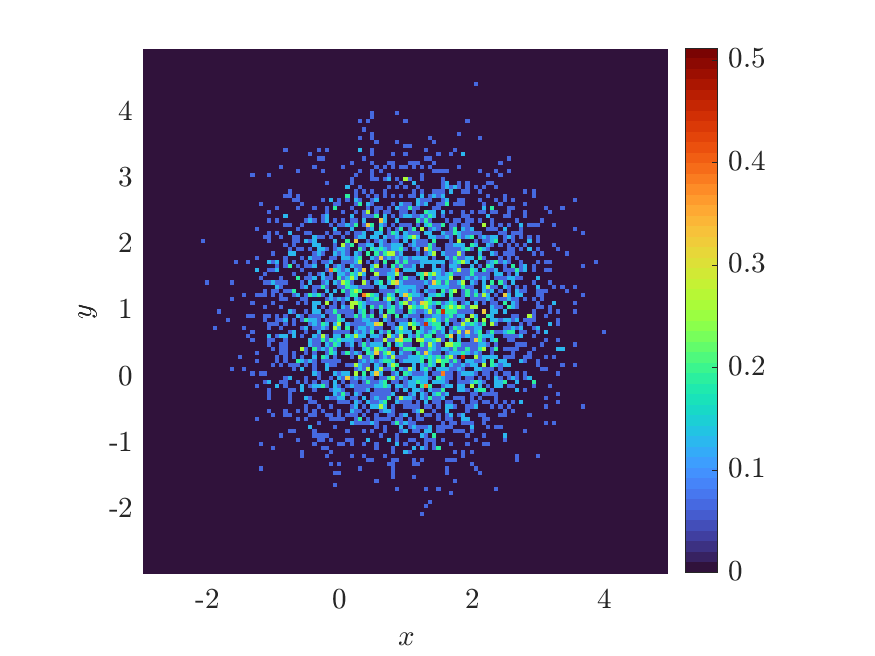} &
\hspace{-0.5cm}	\includegraphics[trim={38 5 40 20},clip,width=0.33\textwidth]{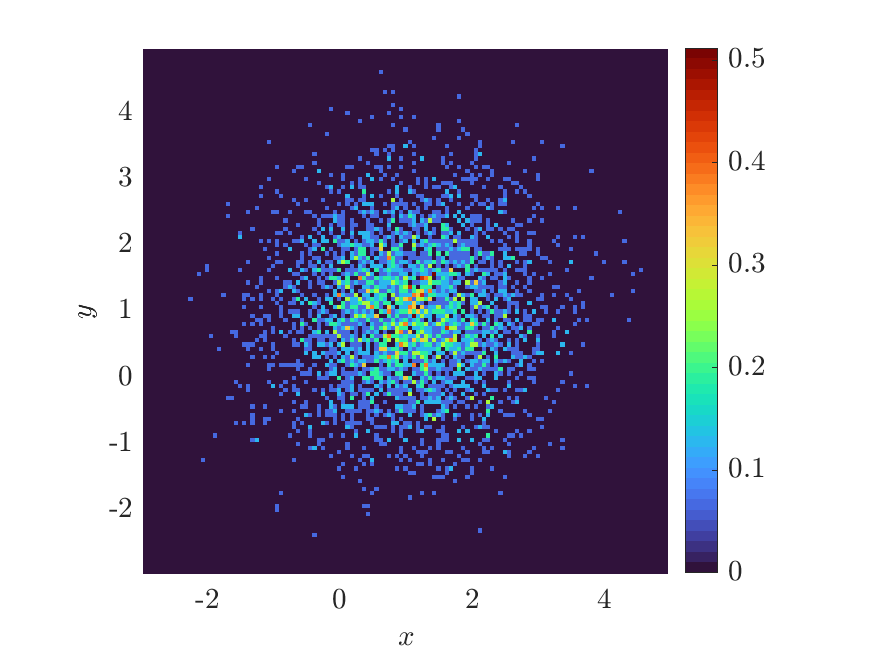} &
\hspace{-0.5cm}	\includegraphics[trim={38 5 40 20},clip,width=0.33\textwidth]{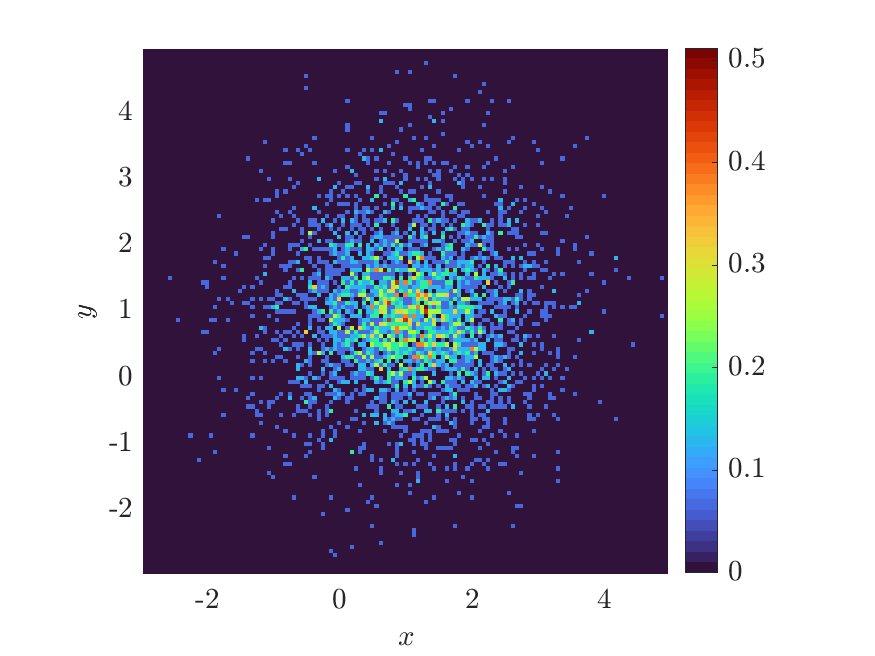}
\end{tabular}
\caption{Histograms created from 4000 particles evolving under logarithmic attraction (equation \eqref{logK}) with varying noise levels at times (left to right) $t=4$, $t=8$, and $t=12$. Top: $\ep= 0.316$, $\sigma=0$ (extrinsic only). Bottom: $\ep=0$, $\sigma = (4\pi)^{-1/2}\approx 0.28$ (intrinsic only).}
\label{log2D_sol}
\end{figure}

We now discuss an example of singular interaction in two spatial dimensions using the logarithmic potential
\begin{equation}\label{logK}
K(x) = \frac{1}{2\pi}\log|x|
\end{equation}
with constant diffusivity $\sigma(x)=\sigma\in\{0,\frac{1}{\sqrt{4\pi}}\}$. This example corresponds to the parabolic-elliptic Keller-Segel model of chemotaxis, where $\sigma_c:= \frac{1}{\sqrt{4\pi}}$ is the critical diffusivity such that $\sigma>\sigma_c$ leads diffusion-dominated spreading of particles throughout the domain (vanishing particle density at every point in $\Rbb^2$) and $\sigma<\sigma_c$ leads to aggregation-dominated concentration of the particle density to the dirac-delta located at the center of mass of the initial particle density \cite{dolbeault2004optimal,carrillo2019existence}. For $\sigma=0$ we examine the affect of additive i.i.d.\ measurement noise $\varepsilon\sim \CalN(0,\ep^2\nrm{\Xbf_\tbf}_{\text{RMS}}^2)$ for $\ep\in \{0.01,0.0316,0.1,0.316,1\}$.\\

We simulate the particle system with a cutoff potential
\begin{equation}\label{logcutoff}
K_\delta(x) = \begin{dcases} \frac{1}{2\pi} \left(\log(\delta)-1+\frac{|x|}{\delta}\right), &|x|<\delta \\ \frac{1}{2\pi} \log|x|, &|x|\geq \delta\end{dcases}
\end{equation}
with $\delta  = 0.01$, so that $K_\delta$ is Lipschitz and $\nabla K_\delta$ has a jump discontinuity at the origin. Initial particle positions are uniformly distributed on a disk of radius 2 and the particle position data consists of $81$ timepoints recorded at a resolution $\Delta t=0.1$, coarsened from $0.0025$. Histograms are created with $128\times 128$ bins in $x$ and $y$ of sidelength $h=0.0469$ (see Figure \ref{log2D_sol} for histogram snapshots over time). We examine $M=2^0,\dots,2^6$ experiments with $N=2000$ or $N=4000$ particles.\\

\begin{figure}
\centering
\begin{tabular}{cc}
	\includegraphics[trim={0 0 25 15},clip,width=0.4\textwidth]{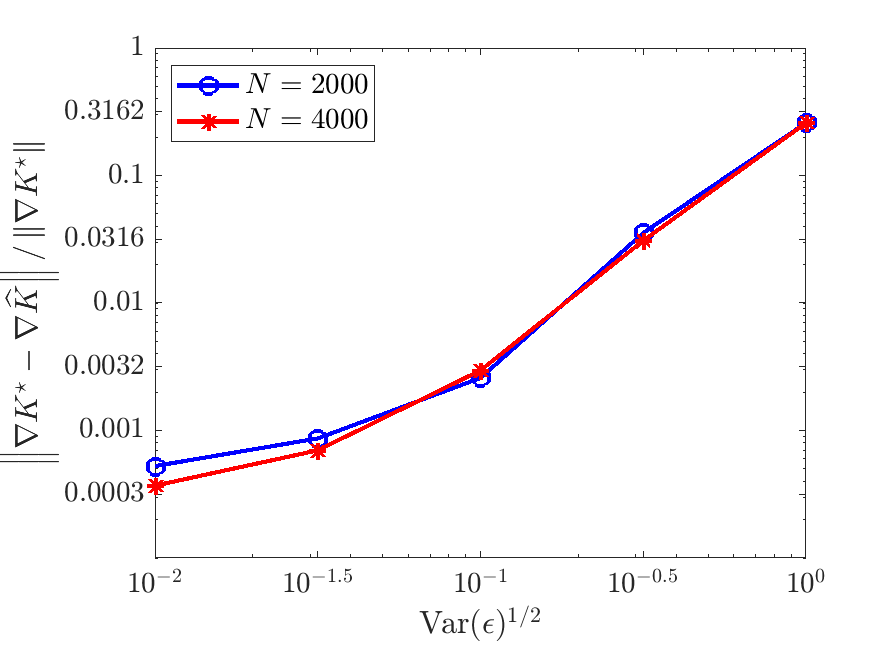} &
\hspace{-0.5cm}	\includegraphics[trim={0 0 25 15},clip,width=0.4\textwidth]{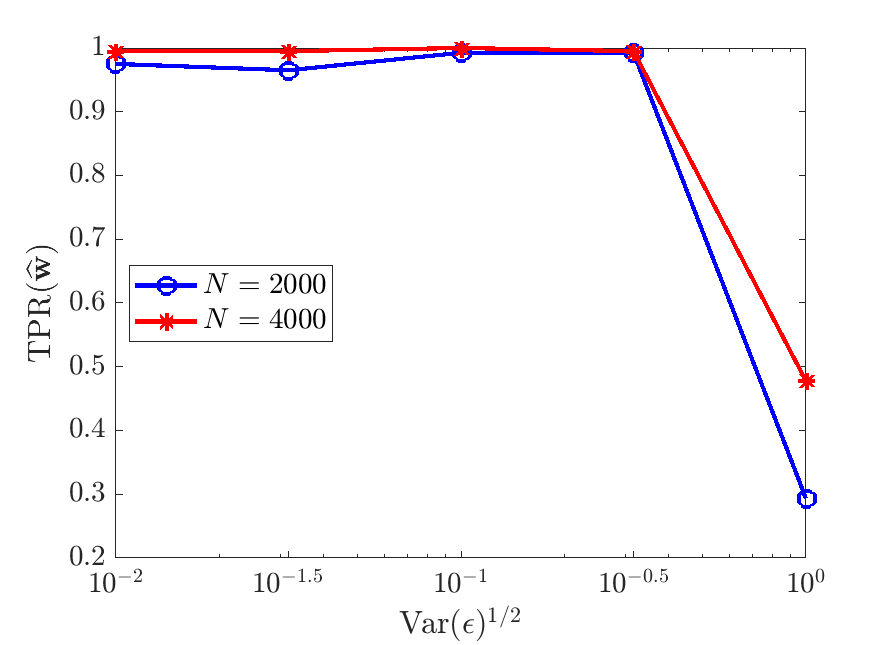}
\end{tabular}
\caption{Recovery of \eqref{fpmeanfield} in two spatial dimensions with $K^\star$ given by \eqref{logK} from deterministic particles ($\sigma^\star=0$) with extrinsic noise $\ep$.}
\label{log2D_nu-Inf}
\end{figure}

In Figure \ref{log2D_nu-Inf} we observe a similar trend in the $\sigma=0$ case as in the 1D nonlocal example, namely that recovery for $\ep\leq 0.1$ is robust with low errors in $\widehat{K}$ (on the order of $0.0032$), only in this case the full model is robustly recovered up to $\ep=0.316$. At $\ep=1$, with $N=4000$ the method frequently identifies a diffusion term $\nu \Delta u$ with $\nu\approx 0.5=\ep^2/2$, and for $N=2000$ the method occasionally identifies the backwards diffusion equation $\partial_t\mu_t= -\alpha \Delta \mu_t$, $\alpha>0$. This is easily prevented by enforcing positivity, which we leave as an extension for future work. \\

 \begin{figure}
\centering
\begin{tabular}{ccc}
	\includegraphics[trim={0 0 25 15},clip,width=0.33\textwidth]{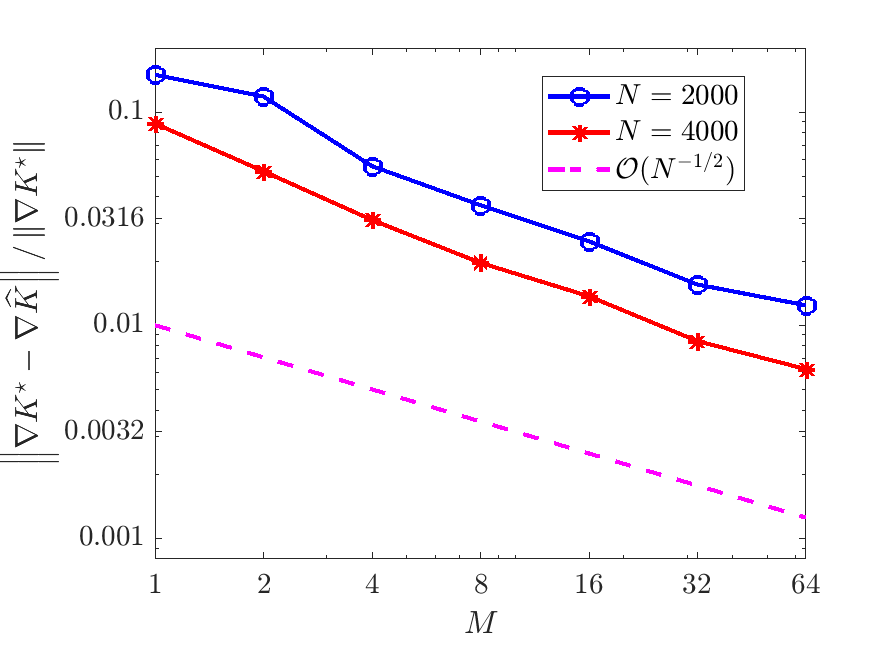} &
\hspace{-0.5cm}	\includegraphics[trim={0 0 25 15},clip,width=0.33\textwidth]{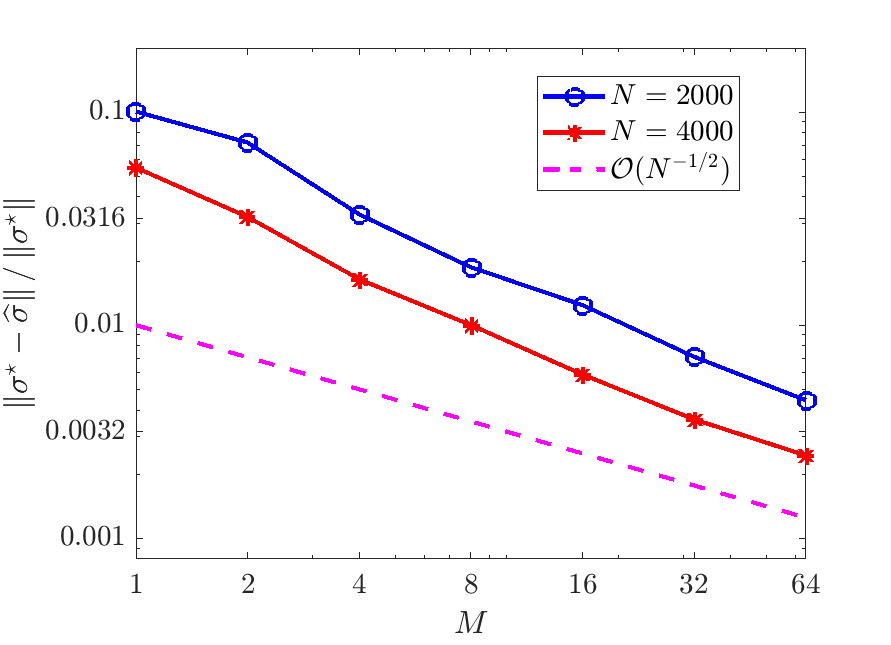} &
\hspace{-0.5cm}	\includegraphics[trim={0 0 25 15},clip,width=0.33\textwidth]{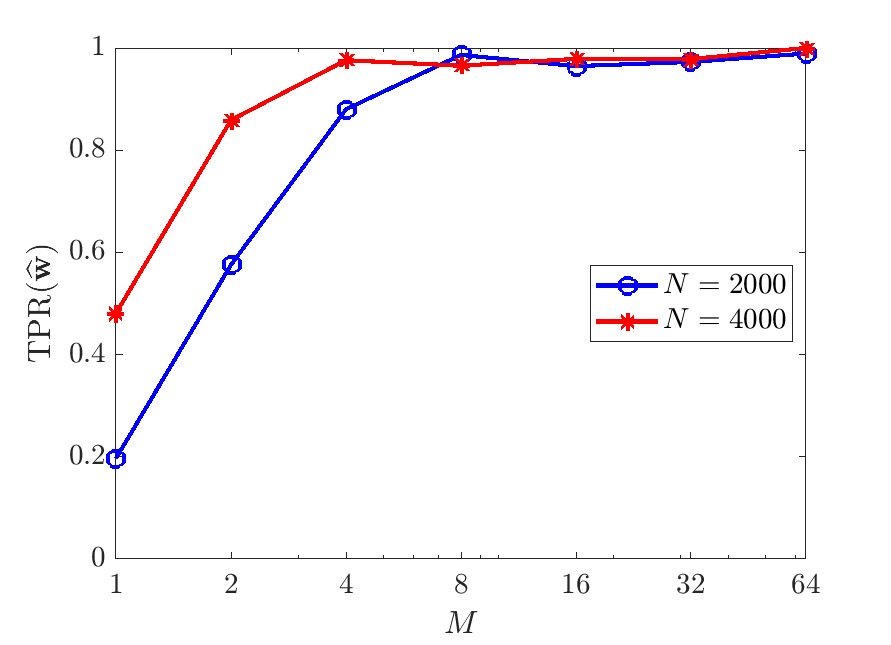}
\end{tabular}
\caption{Recovery of \eqref{fpmeanfield} in two spatial dimensions with $K^\star$ given by \eqref{logK} and $\sigma^\star=\frac{1}{\sqrt{4\pi}}$.}
\label{log2D_nu-1}
\end{figure}

With diffusivity $\sigma=\frac{1}{\sqrt{4\pi}}$, we obtain TPR$(\what)$ approximately greater than 0.95 for $NM\geq 16,000$ (Figure \ref{log2D_nu-1}, right), with an error trend in $\widehat{K}$ following an $\CalO(N^{-1/2})$ rate, and a trend in $\widehat{\sigma}$ of roughly $\CalO(N^{-2/3})$. Since convergence in $M$ for any fixed $N$ is not covered by the theorem above, this shows that combining multiple experiments may yield similar accuracy trends for moderately-sized particle systems. 

\section{Discussion}\label{sec:discussion}

We have developed a weak-form method for sparse identification of governing equations for interacting particle systems using the formalism of mean-field equations. In particular, we have investigating two lines of inquiry, (1) is the mean-field setting applicable for inference from medium-size batches of particles? And (2) can a low-cost, low-regularity density approximation such as a histogram be used to enforce weak-form agreement with the mean-field PDE? We have demonstrated on several examples that the answer is yes to both questions, despite the fact that the mean-field equations are only valid in the limit of infinitely many particles ($N\to \infty$). This framework is suitable for systems of several thousand particles in one and two spatial dimensions, and we have proved convergence in $N$ for the associated least-squares problem using simple histograms as approximate particle densities. In addition, the sparse regression approach allows one to identify the full system, including interaction potential $K$, local potential $V$, and diffusivity $\sigma$.\\

It was initially unclear whether the mean-field setting could be utilized for finite particle batches, hence this can be seen as a proof of concept, with the potential for many improvements and extensions. On the subject of density estimation, histograms lead to piecewise-constant approximations and resulting $\CalO(h)$ errors, hence choosing a density kernel $G$ to achieve high-accuracy quadrature without sacrificing the $\CalO(N)$ runtime of histogram computation seems prudent. The computational grid $\Cbf$ is also a free parameter, and may be optimized in tandem with a quadrature rule. The equally-spaced approach combined with the trapezoidal rule, as applied here, has several advantages, but may need adjustment for higher dimensions. Another obvious improvement would be to enforce convex constraints in the regression problem, such as lower bounds on diffusivity, or $K$ with long-range attraction depending on the distribution $\rho_{rr}\in \CalP([0,\infty))$ of pairwise distances (see \cite{lu2020learning} for further use $\rho_{rr}$). For extensions, the example system \eqref{cos2Deq} and resulting homogenization motivates further study of effective equations for systems with complex microstructure. In other fields this is described as {\it coarse-graining}. A related line of study is inference of 2nd-order particle systems, as explored in \cite{supekar2021learning}, which often lead to an infinite hierachy of mean-field equations. Our weak-form approach may provide a principled method for truncated and closing such hierarchies using particle data.

\section{Acknowledgements}
This research was supported in part by the NSF Mathematical Biology MODULUS grant 2054085, in part by the NSF/NIH Joint DMS/NIGMS Mathematical Biology Initiative grant R01GM126559, and in part by the NSF Computing and Communications Foundations grant 1815983. This work also utilized resources from the University of Colorado Boulder Research Computing Group, which is supported by the National Science Foundation (awards ACI-1532235 and ACI-1532236), the University of Colorado Boulder, and Colorado State University. The authors would also like to thank Prof.\ Vanja Duki\'c (University of Colorado at Boulder, Department of Applied Mathematics) for insightful discussions and helpful suggestions of references.

\bibliographystyle{plain}
\bibliography{researchCU.bib}

\begin{thebibliography}{10}

\bibitem{araujo2019mean}
Dyego Ara{\'u}jo, Roberto~I Oliveira, and Daniel Yukimura.
\newblock A mean-field limit for certain deep neural networks.
\newblock {\em arXiv preprint arXiv:1906.00193}, 2019.

\bibitem{bi2016motility}
Dapeng Bi, Xingbo Yang, M~Cristina Marchetti, and M~Lisa Manning.
\newblock Motility-driven glass and jamming transitions in biological tissues.
\newblock {\em Physical Review X}, 6(2):021011, 2016.

\bibitem{bibby1995martingale}
Bo~Martin Bibby and Michael S{\o}rensen.
\newblock Martingale estimation functions for discretely observed diffusion
  processes.
\newblock {\em Bernoulli}, pages 17--39, 1995.

\bibitem{bishwal2007parameter}
Jaya~PN Bishwal.
\newblock {\em Parameter estimation in stochastic differential equations}.
\newblock Springer, 2007.

\bibitem{bishwal2011estimation}
Jaya Prakash~Narayan Bishwal et~al.
\newblock Estimation in interacting diffusions: Continuous and discrete
  sampling.
\newblock {\em Applied Mathematics}, 2(9):1154--1158, 2011.

\bibitem{blondel2010continuous}
Vincent~D Blondel, Julien~M Hendrickx, and John~N Tsitsiklis.
\newblock Continuous-time average-preserving opinion dynamics with
  opinion-dependent communications.
\newblock {\em SIAM Journal on Control and Optimization}, 48(8):5214--5240,
  2010.

\bibitem{boers2016mean}
Niklas Boers and Peter Pickl.
\newblock On mean field limits for dynamical systems.
\newblock {\em Journal of Statistical Physics}, 164(1):1--16, 2016.

\bibitem{bolley2011stochastic}
Fran{\c{c}}ois Bolley, Jos{\'e}~A Canizo, and Jos{\'e}~A Carrillo.
\newblock Stochastic mean-field limit: non-lipschitz forces and swarming.
\newblock {\em Mathematical Models and Methods in Applied Sciences},
  21(11):2179--2210, 2011.

\bibitem{bongini2017inferring}
Mattia Bongini, Massimo Fornasier, Markus Hansen, and Mauro Maggioni.
\newblock Inferring interaction rules from observations of evolutive systems i:
  The variational approach.
\newblock {\em Mathematical Models and Methods in Applied Sciences},
  27(05):909--951, 2017.

\bibitem{boninsegna2018sparse}
Lorenzo Boninsegna, Feliks N{\"u}ske, and Cecilia Clementi.
\newblock Sparse learning of stochastic dynamical equations.
\newblock {\em The Journal of chemical physics}, 148(24):241723, 2018.

\bibitem{brunton2016discovering}
Steven~L Brunton, Joshua~L Proctor, and J~Nathan Kutz.
\newblock Discovering governing equations from data by sparse identification of
  nonlinear dynamical systems.
\newblock {\em Proceedings of the national academy of sciences},
  113(15):3932--3937, 2016.

\bibitem{callaham2021nonlinear}
Jared~L Callaham, J-C Loiseau, Georgios Rigas, and Steven~L Brunton.
\newblock Nonlinear stochastic modelling with {L}angevin regression.
\newblock {\em Proceedings of the Royal Society A}, 477(2250):20210092, 2021.

\bibitem{carrillo2019existence}
JA~Carrillo, MG~Delgadino, and FS~Patacchini.
\newblock Existence of ground states for aggregation-diffusion equations.
\newblock {\em Analysis and applications}, 17(03):393--423, 2019.

\bibitem{chen2021maximum}
Xiaohui Chen.
\newblock Maximum likelihood estimation of potential energy in interacting
  particle systems from single-trajectory data.
\newblock {\em Electronic Communications in Probability}, 26:1--13, 2021.

\bibitem{chen2021solving}
Xiaoli Chen, Liu Yang, Jinqiao Duan, and George~Em Karniadakis.
\newblock Solving inverse stochastic problems from discrete particle
  observations using the {F}okker--{P}lanck equation and physics-informed
  neural networks.
\newblock {\em SIAM Journal on Scientific Computing}, 43(3):B811--B830, 2021.

\bibitem{dolbeault2004optimal}
Jean Dolbeault and Beno{\^\i}t Perthame.
\newblock Optimal critical mass in the two dimensional {K}eller--{S}egel model
  in r2.
\newblock {\em Comptes Rendus Mathematique}, 339(9):611--616, 2004.

\bibitem{feng2021data}
Jinchao Feng, Yunxiang Ren, and Sui Tang.
\newblock Data-driven discovery of interacting particle systems using gaussian
  processes.
\newblock {\em arXiv preprint arXiv:2106.02735}, 2021.

\bibitem{fetecau2018zero}
Razvan~C Fetecau, Hui Huang, Daniel Messenger, and Weiran Sun.
\newblock Zero-diffusion limit for aggregation equations over bounded domains.
\newblock {\em arXiv preprint arXiv:1809.01763}, 2018.

\bibitem{fetecau2019propagation}
Razvan~C Fetecau, Hui Huang, and Weiran Sun.
\newblock Propagation of chaos for the {K}eller--{S}egel equation over bounded
  domains.
\newblock {\em Journal of Differential Equations}, 266(4):2142--2174, 2019.

\bibitem{fetecau2011swarm}
Razvan~C Fetecau, Yanghong Huang, and Theodore Kolokolnikov.
\newblock Swarm dynamics and equilibria for a nonlocal aggregation model.
\newblock {\em Nonlinearity}, 24(10):2681, 2011.

\bibitem{fetecau2017swarm}
Razvan~C Fetecau and Mitchell Kovacic.
\newblock Swarm equilibria in domains with boundaries.
\newblock {\em SIAM Journal on Applied Dynamical Systems}, 16(3):1260--1308,
  2017.

\bibitem{freedman1981histogram}
David Freedman and Persi Diaconis.
\newblock On the histogram as a density estimator: L2 theory.
\newblock {\em Zeitschrift f{\"u}r Wahrscheinlichkeitstheorie und verwandte
  Gebiete}, 57(4):453--476, 1981.

\bibitem{gkeka2020machine}
Paraskevi Gkeka, Gabriel Stoltz, Amir Barati~Farimani, Zineb Belkacemi, Michele
  Ceriotti, John~D Chodera, Aaron~R Dinner, Andrew~L Ferguson, Jean-Bernard
  Maillet, Herv{\'e} Minoux, et~al.
\newblock Machine learning force fields and coarse-grained variables in
  molecular dynamics: application to materials and biological systems.
\newblock {\em Journal of Chemical Theory and Computation}, 16(8):4757--4775,
  2020.

\bibitem{gomes2019parameter}
Susana~N Gomes, Andrew~M Stuart, and Marie-Therese Wolfram.
\newblock Parameter estimation for macroscopic pedestrian dynamics models from
  microscopic data.
\newblock {\em SIAM Journal on Applied Mathematics}, 79(4):1475--1500, 2019.

\bibitem{guo2021progress}
Jiawei Guo.
\newblock The progress of three astrophysics simulation methods: Monte-carlo,
  pic and mhd.
\newblock In {\em Journal of Physics: Conference Series}, volume 2012, page
  012136. IOP Publishing, 2021.

\bibitem{jabin2017mean}
Pierre-Emmanuel Jabin and Zhenfu Wang.
\newblock Mean field limit for stochastic particle systems.
\newblock In {\em Active Particles, Volume 1}, pages 379--402. Springer, 2017.

\bibitem{jang2020d}
Jun-Gi Jang and U~Kang.
\newblock D-tucker: Fast and memory-efficient tucker decomposition for dense
  tensors.
\newblock In {\em 2020 IEEE 36th International Conference on Data Engineering
  (ICDE)}, pages 1850--1853. IEEE, 2020.

\bibitem{kasonga1990maximum}
Raphael~A Kasonga.
\newblock Maximum likelihood theory for large interacting systems.
\newblock {\em SIAM Journal on Applied Mathematics}, 50(3):865--875, 1990.

\bibitem{keller1971model}
Evelyn~F Keller and Lee~A Segel.
\newblock Model for chemotaxis.
\newblock {\em Journal of theoretical biology}, 30(2):225--234, 1971.

\bibitem{LagergrenNardiniMichaelLavigneEtAl2020ProcRSocA}
John~H. Lagergren, John~T. Nardini, G.~Michael~Lavigne, Erica~M. Rutter, and
  Kevin~B. Flores.
\newblock Learning partial differential equations for biological transport
  models from noisy spatio-temporal data.
\newblock {\em Proc. R. Soc. A.}, 476(2234):20190800, February 2020.

\bibitem{lang2020learning}
Quanjun Lang and Fei Lu.
\newblock Learning interaction kernels in mean-field equations of 1st-order
  systems of interacting particles.
\newblock {\em arXiv preprint arXiv:2010.15694}, 2020.

\bibitem{lelievre2016partial}
Tony Lelievre and Gabriel Stoltz.
\newblock Partial differential equations and stochastic methods in molecular
  dynamics.
\newblock {\em Acta Numerica}, 25:681--880, 2016.

\bibitem{li2021extracting}
Yang Li and Jinqiao Duan.
\newblock Extracting governing laws from sample path data of non-gaussian
  stochastic dynamical systems.
\newblock {\em arXiv preprint arXiv:2107.10127}, 2021.

\bibitem{lo1988maximum}
Andrew~W Lo.
\newblock Maximum likelihood estimation of generalized it{\^o} processes with
  discretely sampled data.
\newblock {\em Econometric Theory}, 4(2):231--247, 1988.

\bibitem{lu2020learning}
Fei Lu, Mauro Maggioni, and Sui Tang.
\newblock Learning interaction kernels in heterogeneous systems of agents from
  multiple trajectories.
\newblock {\em J. Mach. Learn. Res.}, 22:32--1, 2021.

\bibitem{lukeman2010inferring}
Ryan Lukeman, Yue-Xian Li, and Leah Edelstein-Keshet.
\newblock Inferring individual rules from collective behavior.
\newblock {\em Proceedings of the National Academy of Sciences},
  107(28):12576--12580, 2010.

\bibitem{malik2018low}
Osman~Asif Malik and Stephen Becker.
\newblock Low-rank tucker decomposition of large tensors using tensorsketch.
\newblock {\em Advances in neural information processing systems},
  31:10096--10106, 2018.

\bibitem{meleard1996asymptotic}
Sylvie M{\'e}l{\'e}ard.
\newblock Asymptotic behaviour of some interacting particle systems;
  mckean-vlasov and boltzmann models.
\newblock In {\em Probabilistic models for nonlinear partial differential
  equations}, pages 42--95. Springer, 1996.

\bibitem{messenger2020weakpde}
Daniel~A Messenger and David~M Bortz.
\newblock Weak {SIND}y for partial differential equations.
\newblock {\em Journal of Computational Physics}, page 110525, 2021.

\bibitem{messenger2020weak}
Daniel~A Messenger and David~M Bortz.
\newblock Weak {SIND}y: Galerkin-based data-driven model selection.
\newblock {\em Multiscale Modeling \& Simulation}, 19(3):1474--1497, 2021.

\bibitem{messenger2020equilibria}
Daniel~A Messenger and Razvan~C Fetecau.
\newblock Equilibria of an aggregation model with linear diffusion in domains
  with boundaries.
\newblock {\em Mathematical Models and Methods in Applied Sciences},
  30(04):805--845, 2020.

\bibitem{milstein1994numerical}
Grigorii~Noikhovich Milstein.
\newblock {\em Numerical integration of stochastic differential equations},
  volume 313.
\newblock Springer Science \& Business Media, 1994.

\bibitem{nardini2021learning}
John~T Nardini, Ruth~E Baker, Matthew~J Simpson, and Kevin~B Flores.
\newblock Learning differential equation models from stochastic agent-based
  model simulations.
\newblock {\em Journal of the Royal Society Interface}, 18(176):20200987, 2021.

\bibitem{rudy2017data}
Samuel~H Rudy, Steven~L Brunton, Joshua~L Proctor, and J~Nathan Kutz.
\newblock Data-driven discovery of partial differential equations.
\newblock {\em Science Advances}, 3(4):e1602614, 2017.

\bibitem{sepulveda2013collective}
N{\'e}stor Sep{\'u}lveda, Laurence Petitjean, Olivier Cochet, Erwan
  Grasland-Mongrain, Pascal Silberzan, and Vincent Hakim.
\newblock Collective cell motion in an epithelial sheet can be quantitatively
  described by a stochastic interacting particle model.
\newblock {\em PLoS computational biology}, 9(3):e1002944, 2013.

\bibitem{sharrock2021parameter}
Louis Sharrock, Nikolas Kantas, Panos Parpas, and Grigorios~A Pavliotis.
\newblock Parameter estimation for the mckean-vlasov stochastic differential
  equation.
\newblock {\em arXiv preprint arXiv:2106.13751}, 2021.

\bibitem{sun2020low}
Yiming Sun, Yang Guo, Charlene Luo, Joel Tropp, and Madeleine Udell.
\newblock Low-rank tucker approximation of a tensor from streaming data.
\newblock {\em SIAM Journal on Mathematics of Data Science}, 2(4):1123--1150,
  2020.

\bibitem{supekar2021learning}
Rohit Supekar, Boya Song, Alasdair Hastewell, Alexander Mietke, and J{\"o}rn
  Dunkel.
\newblock Learning hydrodynamic equations for active matter from particle
  simulations and experiments.
\newblock {\em arXiv preprint arXiv:2101.06568}, 2021.

\bibitem{sznitman1991topics}
Alain-Sol Sznitman.
\newblock Topics in propagation of chaos.
\newblock In {\em Ecole d'{\'e}t{\'e} de probabilit{\'e}s de Saint-Flour
  XIX—1989}, pages 165--251. Springer, 1991.

\bibitem{van2015simulating}
Paul Van~Liedekerke, MM~Palm, N~Jagiella, and Dirk Drasdo.
\newblock Simulating tissue mechanics with agent-based models: concepts,
  perspectives and some novel results.
\newblock {\em Computational particle mechanics}, 2(4):401--444, 2015.

\bibitem{warren1992astrophysical}
Michael~S Warren and John~K Salmon.
\newblock Astrophysical n-body simulations using hierarchical tree data
  structures.
\newblock {\em Proceedings of Supercomputing}, 1992.

\bibitem{weinan2011principles}
E~Weinan.
\newblock {\em Principles of multiscale modeling}.
\newblock Cambridge University Press, 2011.

\end{thebibliography}

\appendix

\section{Notation \& Specifications for Examples}\label{app:specs}

{\small \begin{table}
\begin{center}
\begin{tabular}{|p{2.8cm}p{10cm}p{3.5cm}|}
\hline Variable & Definition & Domain \\ \hline
$K$ & pairwise interaction potential & $L^1_{loc}(\Rbb^d,\Rbb)$ \\ \hline
$V$ & local potential & $C(\Rbb^d,\Rbb)$ \\ \hline
$\sigma$ & diffusivity & $C(\Rbb^d,\Rbb^{d\times d})$ \\ \hline
$N$ & number of particles per experiment & $\{2,3,\dots\}$ \\ \hline
$d$ & dimension of latent space & $\Nbb$ \\ \hline
$T$ & final time & $(0,\infty)$ \\ \hline
$(\Omega,\CalB,\Pbb,(\CalF_t)_{t\geq 0})$ & filtererd probability space & \\  \hline
$(B^{(i)}_t)_{i=1}^N$ & independent $\Rbb^d$ Brownian motions on $(\Omega,\CalB,\Pbb,(\CalF_t)_{t\geq 0})$ & \\  \hline
$X^{(i)}_t$ & $i$th particle in the particle system \eqref{dXt} at time $t$ & $\Rbb^d$ \\ \hline
$\Xbf_t$ & $N$-particle system \eqref{dXt} at time $t$ & $\Rbb^{Nd}$ \\ \hline
$\mu^N_t$ & empirical measure & $\CalP(\Rbb^d)$\\ \hline
$F^N_t$ & distribution of the process $\Xbf_t$ in $\Rbb^{Nd}$ & \\ \hline
$X_t$ & mean-field process \eqref{mckeanvlasovSDE} at time $t$ & $\Rbb^{Nd}$ \\ \hline
$\mu_t$ & distribution of $X_t$ & $\CalP(\Rbb^d)$\\ \hline
$\tbf$ & $L$ discrete timepoints & $[0,T]$ \\ \hline
$\pmb{\Xbb}_\tbf$ & Collection of $M$ independent samples of $\Xbf_t$ at $\tbf$ & $\Rbb^{MLNd}$ \\ \hline
$\pmb{\Ybb}_\tbf$ & Sample of $\Xbf_{\tbf}$ corrupted with i.i.d.\ additive noise & $\Rbb^{MLNd}$ \\ \hline
$U_t$ & approximate density from particle positions & $\CalP(\Rbb^d)$ \\ \hline
$G$ & density kernel mapping $\mu^N_t$ to $U_t$ & $L^1(\Rbb^d\times\Rbb^d, \Rbb)$\\\hline
$\CalD$ & spatial support of $U_t$, $t\in[0,T]$ & compact subset of $\Rbb^d$ \\\hline
$\Cbf$ & discretization of $\CalD$ &  \\\hline
$\Ubf_t$ & discrete approximate density $U_t(\Cbf)$ &  \\\hline
$\lan \cdot,\cdot\ran_h$ & \raggedright semi-discrete inner product, trapezoidal rule over $\Cbf$ &  \\\hline
$\lan \cdot,\cdot\ran_{h,\Delta t}$ & fully-discrete inner product, trapezoidal rule  over $\Cbf\times \tbf$ &  \\\hline
$\Lbb_K$ & library of candidate interaction forces &  \\ \hline
$\Lbb_V$ & library of candidate local forces &  \\ \hline
$\Lbb_\sigma$ & library of candidate diffusivities &  \\ \hline
$\Lbb$ & $(\Lbb_K,\Lbb_V,\Lbb_\sigma)$ &  \\ \hline
$\Psi$ & set of $n$ test functions $(\psi_k)_{k=1}^n$ & \\ \hline
$\phi_{m,p}(v;\, \Delta)$ & test functions used in this work (equation \eqref{testfcn}) & \\  \hline
$\pmb{\lambda}$ & set of sparsity thresholds &  \\  \hline
$\CalL$ & loss function for sparsity thresholds (equation \eqref{lossfcn}) & \\ \hline
\end{tabular}
\end{center}
\caption{Notations used throughout.}
\end{table}}

\begin{table}
\begin{center}
\begin{tabular}{|l|l|}
\hline
Mean-field Term & Trial Function Library  \\\hline
$\nabla \cdot(U\nabla K*U)$ & $\nabla \cdot(U\nabla |x|^m*U)$, $m\in \{1, 2,3,4,5,6, 7\}$ \\ \hline
$\nabla \cdot(U\nabla V)$ & $\partial_{x_i}\left(U\cos(m x_1)\cos(n x_2)\right)$, $(m,n)\in\{0,1,2,3,4,5\}$, $i\in\{1,2\}$ \\\hline
$\frac{1}{2}\sum_{i=1}^d \frac{\partial^2(U\sigma\sigma^T)_{ij}}{\partial x_i\partial x_j}$ & $\Delta (U\cos(m x_1)\cos(n x_2))$, $(m,n)\in\{0,1,2,3,4,5\}$  \\ \hline
\end{tabular}
\end{center}
\caption{Trial function library for local 2D example (Section \ref{sec:cos2D}).}\label{coslib}
\end{table}

\begin{table}
\begin{center}
\begin{tabular}{|l|l|}
\hline
Mean-field Term & Trial Function Library  \\\hline
$\nabla \cdot(U\nabla K*U)$ & $\partial_x \cdot(U\partial_x  |x|^m*U)$, $m\in \{1,2,3,4,5,6,7\}$ \\ \hline
$\nabla \cdot(U\nabla V)$ & $\partial_x \left(U x^m \right)$, $m\in \{0,2,3,4,5,6,7,8\}$ \\ \hline
$\frac{1}{2}\sum_{i=1}^d \frac{\partial^2(U\sigma\sigma^T)_{ij}}{\partial x_i\partial x_j}$ & $\partial_{xx} (Ux^m)$, $m\in\{0,1,2,3,4,5,6,7,8\}$  \\ \hline
\end{tabular}
\end{center}
\caption{Trial function library for nonlocal 1D example (Section \ref{sec:qanr}).}\label{qanrlib}
\end{table}

\begin{table}
\begin{center}
\begin{tabular}{|l|l|}
\hline
Mean-field Term & Trial Function Library  \\\hline
$\nabla \cdot(U\nabla K*U)$ & $\begin{dcases} \nabla \cdot(U\nabla  |x|^m*U), & m\in \{2,3,4,5,6\} \\ \nabla \cdot(U\nabla  \left[|x|^{1/2}\right]_\delta*U) \\
\nabla \cdot(U\nabla  \left[|x|(\log|x|-1)\right]_\delta*U)\\
\nabla \cdot(U\nabla  \left[\log|x|\right]_\delta*U)\\
\end{dcases}$\\ \hline
$\nabla \cdot(U\nabla V)$ & $\nabla \left(U \nabla (x_1^mx_2^n) \right)$, $(m,n)\in\Nbb\times\Nbb, 1\leq m+n\leq 6$ \\ \hline
$\frac{1}{2}\sum_{i=1}^d \frac{\partial^2(U\sigma\sigma^T)_{ij}}{\partial x_i\partial x_j}$ & $\frac{\partial^2}{\partial x_i\partial x_j}(U\cos(m x_1)\cos(n x_2))$, $(m,n)\in\{0,1,2\}$,\,$(i,j)\in\{1,2\}$ \\ \hline
\end{tabular}
\end{center}
\caption{Trial function library for nonlocal 2D example (Section \ref{sec:log2D}). Interaction potentials $[\,K\,]_\delta$ indicate cutoff potentials of the form \eqref{logcutoff} with $\delta=0.01$ such that the resulting potential is Lipschitz.}\label{coslib}
\end{table}

\begin{table}
\begin{center}
\hspace{-2cm}\begin{tabular}{|c|c|c|c|c|c|c|c|c|c|c|c|}
\hline
Example & $m_x$ & $m_t$ & $p_x$ & $p_t$ & $s_x$ & $s_t$ &  size$(\Gbf)$ & $\kappa(\Gbf)$ & size$(\Ubf)$ & $(h,\Delta t)$ & Walltime \\ \hline
Local 2D & 31 & 16 & 5 & 3 & 10 & 5& 686$\times 85$ & $3.8\times 10^7$ & $128\times 128 \times 101$ & $(0.078,0.02)$& 9.7s\\\hline
Nonlocal 1D & 29 & 8 & 5 & 3 & 5 & 1  & 3368$\times 24$ & $8.7\times 10^8$ & $256\times 101$ & $(0.023,0.01)$ & 2.6s\\\hline
Nonlocal 2D & 25 & 8 & 5 & 3 & 8 & 1  & 6500$\times 53$ & $4.8\times 10^7$ & $128\times 128\times 81$ & $(0.047,0.1)$ & 8.5s\\\hline
\end{tabular}
\end{center}
\caption{Discretization parameters for examples. Note the condition number $\kappa(\Gbf)$ and walltime are listed for representative samples with 64,000 total particles.}
\label{specs}
\end{table}


\end{document}